\documentclass{article}
\usepackage{arxiv}

\usepackage[utf8]{inputenc} 
\usepackage[T1]{fontenc}    
\usepackage{hyperref}       
\usepackage{url}            
\usepackage{booktabs}       
\usepackage{amsfonts}       
\usepackage{graphicx}
\usepackage[authoryear]{natbib}
\usepackage{algorithm}
\usepackage{algpseudocode}
\usepackage{amsmath}

\usepackage{newfloat}
\usepackage{listings}
\floatstyle{ruled}
\newfloat{listing}{tb}{lst}{}
\floatname{listing}{Listing}

\usepackage{url}            
\usepackage{booktabs}       
\usepackage{amsfonts}       
\usepackage{nicefrac}       
\usepackage{microtype}      
\usepackage{xcolor}         

\usepackage{graphicx}
\usepackage{algpseudocode}
\usepackage{amsmath}

\usepackage{tikz}
\usetikzlibrary{arrows, automata}
\usetikzlibrary{shapes.geometric, positioning}

\usepackage[switch]{lineno}
\usepackage{amsthm}

\newtheorem{theorem}{Theorem}

\algnewcommand{\LeftComment}[1]{\State \(\triangleright\) #1}

\title{Directed Graph Auto-Encoders}

\date{}
\author{
  Georgios Kollias, 
  Vasileios Kalantzis, 
  Tsuyoshi Id\'e, 
  Aur\'elie Lozano, 
  Naoki Abe \\
  IBM Research\\
  T. J. Watson Research Center\\
  \texttt{\{gkollias, vkal, tide, aclozano, nabe\}@us.ibm.com}
}

\renewcommand{\headeright}{}
\renewcommand{\undertitle}{}

\begin{document}
\maketitle

\begin{abstract}
We introduce a new class of auto-encoders for directed graphs, motivated by a direct extension of the Weisfeiler-Leman algorithm to pairs of node labels. The proposed model learns pairs of interpretable latent representations for the nodes of directed graphs, and uses parameterized graph convolutional network (GCN) layers for its  encoder and an asymmetric inner product decoder. Parameters in the encoder control the weighting of representations exchanged between neighboring nodes. We demonstrate the ability of the proposed model to learn meaningful latent embeddings and achieve superior performance on the directed link prediction task on several popular network datasets.
\end{abstract}

\section{Introduction}
Graph-structured data are ubiquitous, commonly encountered in diverse domains, ranging from biochemical interaction networks, to networks of social and economic transactions. A graph introduces dependencies between its connected nodes, thus algorithms designed to work solely with feature vectors of isolated nodes as inputs can yield suboptimal results. One way to remedy this issue, without reverting to more complex graph algorithms, is to enhance the representation of a graph node so that both its features and embedding graph structure are captured in a single vector.

Graph Convolutional Networks (GCNs) produce vectorial representations of nodes that are graph-aware and have been successfully used in downstream learning tasks including node classification, link prediction and graph classification. The construction of GCNs falls into two categories: spatial-based and spectral-based. Spatial-based GCNs are conveniently described as Message Passing Neural Networks (MPNNs) detailing the steps for aggregating information from neighbor graph nodes \cite{gilmer2017neural, micheli2009neural, niepert2016learning}. They adopt a local view of the graph structure around each node and are straightforward to describe and lightweight to compute; however they also need local customizations to enhance the performance of the representations they produce for downstream learning tasks \cite{velikovi2017graph}.  Spectral-based GCNs originate in graph signal processing perspectives \cite{bruna2014spectral}. They are based on the graph Laplacian, so they inherently adopt a global graph view
\cite{defferrard2016convolutional, kipf2016semi}. However they incur more computational cost, typically addressed by approximating their convolutional filter.

In this work we focus on GCN-based models for representing the nodes of directed graphs (encoding), so that we can faithfully reconstruct their directed edges (decoding). Our goal is to identify both whether two nodes $u$ and $v$ should be connected or not (which is the only goal for the undirected case) and whether their connection has the direction $u \mapsto v$ or $v \mapsto u$ or both. Many applications depend critically on this distinctionality. For example in (directed) citation graphs it cannot be the case that a publication cites a work that is published later in time, in (directed) causal graphs a causal node should prepend any of its effects, in knowledge graphs subject nodes are expected to point to object nodes. As a consequence, failing to identify the correct orientation even in a single edge could severely disrupt downstream tasks: time ordering can become contradicting, paths to root causes can be erroneously blocked, flow computations between entities can be totally wrong.

Our encoder follows a color refinement scheme for directed graphs that reduces to the standard Weisfeiler-Leman algorithm.
Coupled with an unsymmetric decoder, our directed graph auto-encoder can accurately infer missing directed links from the limited, incomplete graph it has access to during training time. 

Our contributions are two-fold. First, we propose a novel variant of the Weisfeiler-Leman algorithm that clearly emphasizes the dual role of directed graph nodes as both sources and targets of directed links. This abstracts, for the first time, the alternating update of authority and hub scalar values in HITS, and the relations between left and right singular vectors computed in SVD and the GCN-based approaches for directed graphs inspired by them. Second, we design parameterized GCN layers for updating the pair-of-vectors representation of the source and target roles of directed graph nodes in an alternating manner, and use these layers as the encoder in a directed graph auto-encoder architecture. We demonstrate that the parameterization we introduce is performance-critical for learning latent representations, and the proposed model can outperform state-of-the-art methods for 
the {\em directed} link prediction task 
on several popular citation network datasets in terms of area under the ROC curve (AUC) and average precision (AP) metrics.

The code is available at \url{https://github.com/gidiko/DiGAE}.

\section{Preliminaries and background}
We can generally describe the dependencies of a network as a directed graph $G(V, E, w)$, i.e., a weighted \emph{dependency graph}. Here $V$ is the set of $n=|V|$ graph nodes and $E = \{(i, j) \in V \times V: i \mapsto j\}$ is the set of its $m=|E|$ directed edges, expressed as node pairs. Finally $w: V \times V \rightarrow \mathbb{R}$ is the edge weight function, with $w(i, j)$ being a scalar capturing the ``strength'' of the dependency $i \mapsto j$ \emph{iff} $(i, j) \in E$ - and vanishing otherwise. 
Following a linear algebra perspective, we will represent $G(V, E, w)$ as an $n \times n$ sparse, weighted, adjacency matrix $\textbf{A}$. This matrix has $m$ non-vanishing entries and its $(i,j)$ entry is set equal to the respective weight $w(i, j)$, i.e.,  $\textbf{A}\texttt{[i, j]} = w(i, j)$.
Throughout the rest of this paper, we use $\tilde{\mathcal{N}}^{+}(i) = \mathcal{N}^{+}(i) 
\cup \{i\}$ ($\tilde{\mathcal{N}}^{-}(i) = \mathcal{N}^{-}(i) \cup \{i\}$) and
$\texttt{deg}^{+}(i) = |\tilde{\mathcal{N}}^{+}(i)|$ ($\texttt{deg}^{-}(i) = |\tilde{\mathcal{N}}^{-}(i)|$), to denote the neighbor node sets of the outgoing (incoming) edges and outgoing (incoming) degrees of a node $i$  - including itself. Analogously, 
the corresponding diagonal matrices with the outdegrees (indegrees) along their diagonal will be denoted as $\tilde{\textbf{D}}^{+}$ ($\tilde{\textbf{D}}^{-}$), where we top with a tilde mark the names of adjacency matrices with added self-links.   
Note that for undirected, unweighted graphs, the corresponding 
adjacency matrices are symmetric and binary, i.e., $\tilde{\textbf{D}}$ is a diagonal matrix where $\tilde{d}_{ii}$ is the original degree of node $i$ increased by $1$ (because of the added self-link).

\subsection{Dual vector encoding of directed graph nodes}

Consider a directed edge $i \mapsto j$, with weight $w(i, j)$, where  $i$ is the source node and $j$ is the target node. Now, let's assume that node $i$ is equipped with a pair of vectors in $\mathbb{R}^{k}$, $1\leq i \leq n$: (i) vector $\textbf{s}_i$ encodes $i$'s role as a source, which is the same for any of the directed edges it participates as a source, and (ii) vector $\textbf{t}_i$ encodes $i$'s role as a target; similarly for node $j$. The similarity of nodes $i$ and $j$ in building the weighted directed edge $i \mapsto j$ could then be captured by a similarity function $\texttt{sim}(\cdot, \cdot)$ which ideally evaluates to the true edge weight: 
$\texttt{sim}(\textbf{s}_i, \textbf{t}_j)=w(i,j)$.

An immediate choice for the similarity function is the dot product: $\texttt{sim}(\textbf{s}_i, \textbf{t}_j) = \textbf{s}_i^{\top} \textbf{t}_j$, with the encodings originally realized as column vectors. We then can compactly evaluate the weights of all edges, by buliding 
two matrices, $\textbf{S}$ and $\textbf{T}$, where $\textbf{S}\texttt{[i, :]} = \textbf{s}_i^{\top}$ and $\textbf{T}\texttt{[j, :]} = \textbf{t}_j^{\top}$, for all $1 \leq i, j \leq n$, and requiring $\textbf{A} = \textbf{S} \textbf{T}^{\top}$.
This particular choice has been explored in \cite{ou2016asymmetric}, which leverages the 
singular value decomposition (\texttt{SVD}) 
$\textbf{A} = \textbf{U} \mathbf{\Sigma} \textbf{V}^\top$ and sets $\textbf{S} = \textbf{U} \mathbf{\Sigma}^{\frac{1}{2}}$
and $\textbf{T} = \textbf{V} \mathbf{\Sigma}^{\frac{1}{2}}$. The authors also explore a 
variant based on truncated SVD. 

\subsection{Weisfeiler-Leman (WL) algorithm and connection to GCNs}
\subsubsection{One-dimensional Weisfeiler-Leman (1-WL)}
One dimensional Weisfeiler-Leman (1-WL)  algorithm is a well-studied approach for assigning distinct 
labels to the nodes of undirected, unweighted graphs with different topological roles \cite{weisfeiler1968reduction}. Given adjacency information in the form of neighborhood 
lists $\mathcal{N}(i) = [k: i \mapsto k \lor k \mapsto i]$, 
$\forall i \in [0, n)$,  1-dim WL iteratively updates a node's label by computing a bijective hash of its neigbors' labels, and mapping this to a unique label. The procedure 
terminates when the hash-and-map procedure stabilizes. A detailed sketch of the 
1-WL algorithm is provided in the Appendix.

\subsubsection{Graph Convolutional Network (GCN)}
Graph Convolutional Networks (GCNs) suggest convolution operators for graph signals defined 
over graph nodes. 
Following the early analysis in \cite{hammond2011wavelets} and the expansions in \cite{defferrard2016convolutional}, GCNs were particularly popularized by Kipf et al. \cite{kipf2016semi}. In particular, the work in \cite{kipf2016semi} focused on undirected, unweighted graphs, such that the adjacency matrices are symmetric and binary.
Each node $i$ is initially assumed encoded by a $k=k_0$ dimensional vector $\textbf{x}_i$, so all node encodings can be collected in an $n \times k$ 
matrix $\textbf{X} = \textbf{X}^{(0)}$. The goal is to \emph{transform} the node embeddings so 
that a downstream task such as node classification is more accurate. 

The proposed transformation contains a succession of \emph{graph convolutional layers} of the form:
\begin{equation}
\textbf{Z}^{(t)} \leftarrow {\tilde{\textbf{D}}^{-\frac{1}{2}}} \tilde{\textbf{A}} {\tilde{\textbf{D}}^{-\frac{1}{2}}} \textbf{X}^{(t)} \textbf{W}^{(t)}
\end{equation}
interspersed with \emph{nonlinear layers}, like $\texttt{ReLU}$ and $\texttt{softmax}$.
The quantity $\textbf{W}^{(t)}$ is a \emph{learnable} $k_t \times k_{t+1}$ matrix of weights for the $t^{th}$ \emph{graph convolutional layer} ($t=0,1, \ldots$). 
The algorithm in \cite{kipf2016semi} implements the transformation of the 
original encodings $\textbf{X}$:
\begin{equation}
\textbf{Z} \leftarrow \texttt{softmax}(\hat{\textbf{A}}\;\texttt{ReLU}(\hat{\textbf{A}}\;\textbf{X}\;\textbf{W}^{(0)})\;\textbf{W}^{(1)}),
\end{equation}
where $\hat{\textbf{A}}:= {\tilde{\textbf{D}}^{-\frac{1}{2}}} \tilde{\textbf{A}} {\tilde{\textbf{D}}^{-\frac{1}{2}}}$. 

\subsubsection{Connecting 1-WL to GCNs}
In \cite{morris2019weisfeiler} the connection of 1-WL to 1-GNNs is explored. Their basic 1-GNN model assumes the form
\begin{equation} \label{eq:1-GNN}
f^{(t)}(v) = \sigma\left( f^{(t-1)}(v)\cdot W_1^{(t)} + \sum_{w \in N(v)} f^{(t-1)}(w)\cdot W_2^{(t)}\right)
\end{equation}
In this, $f^{(t)}(v)$, which is the row feature vector of node $v$ at layer $t > 0$, is computed by aggregating its feature vector and the feature vectors of its neighbors at the previous layer $t-1$, after first multiplying them respectively by parameter matrices $W_1^{(t)}$ and $W_2^{(t)}$. It follows GCNs are 1-GNNs.
These results establish the connection of 1-WL to 1-GNNs (\ref{eq:1-GNN}):
\begin{itemize}
\item
  \textbf{[Theorem 1 in \cite{morris2019weisfeiler}]} For all $t \geq 0$ and for all choices of weights $\textbf{W}^{(t)} = {(W_1^{(t^{'})}, W_2^{(t^{'})})}_{t^{'} \leq t}$, coloring $c_{l}^{(t)}$ refines encoding $f^{(t)}$: $c_{l}^{(t)} \sqsubseteq f^{(t)}$. This means that for any nodes $u, w$ in $G_{st}$, $c_{l}^{(t)}(u) = c_{l}^{(t)}(w)$ implies $f^{(t)}(u) = f^{(t)}(w)$.
\item
  \textbf{[Theorem 2 in \cite{morris2019weisfeiler}]}
  For all $t \geq 0$  there exists a sequence of $\textbf{W}^{(t)}$ and a 1-GNN architecture such that the colorings and the encodings are equivalent: $c_{l}^{(t)} \equiv f^{(t)}$ (i.e. $c_{l}^{(t)} \sqsubseteq f^{(t)}$ and $f^{(t)} \sqsubseteq c_{l}^{(t)}$).
\end{itemize}

\section{DiGAE: Directed Graph Auto-Encoder}

\subsection{Weisfeiler-Leman (WL) algorithm and connection to encoder layers in DiGAE}
We now extend 1-WL for coloring pairs of node labels in directed graphs. We can formally prove that this extension reduces to standard 1-WL over a bipartite, undirected graph. This allows us to connect it to 1-GNNs with a special structure for the parameter matrices, which corresponds to the directed graph convolutional layers in the encoder module of our DiGAE architecture which we define next. For the case of shared weight matrices we can even remove this special structure requirement.    

\subsubsection{Coloring pairs of node labels in directed graphs}
To extend 1-WL for directed graphs, we equip each node of the graph with $2$ labels, one for capturing its role as a ``source'' of directed edges emanating from it and another one for its role as a ``target'' node for directed edges pointing to it. Then at each step, we propose the ``source'' label of a node to be a bijective function of its current ``source'' label and the multiset of ``target'' labels of the nodes it points to. In parallel, its ``target'' label will be updated to the bijective mapping of its current target label and the ``source'' labels of the nodes pointing to it. For the detailed algorithm, please refer to the Appendix.  
\subsubsection{Neural Source and Target encodings}
The propagation of ``source'' and ``target'' node labels along their directed edges, as in 
the extension of the 1-WL algorithm, suggests a new graph convolutional 
network layer. Similarly to the way the GCN layer is closely connected to standard 1-WL, 
we propose a novel \emph{directed graph convolutional layer} that transforms a pair of source 
and target encoding vectors for each node, in analogy to our extension for directed graphs. 
We assume that each node $i$ is originally encoded by a pair of vectors, its \emph{source} $\textbf{s}_i^{(0)}$ and \emph{target} $\textbf{t}_i^{(0)}$ encodings, and collect these encodings as rows in matrices $\textbf{S}^{(0)}$ and  $\textbf{T}^{(0)}$.

We can now define the $t^{th}$ graph convolutional layer for updating the \emph{source} encodings by aggregating the transformed and normalized \emph{target} encodings in the neighborhood as:
\begin{equation}
  \textbf{S}^{(t+1)} \leftarrow \left({\tilde{\textbf{D}}^{+}}\right)^{-\beta} \tilde{\textbf{A}}  \left({\tilde{\textbf{D}}^{-}}\right)^{-\alpha} \textbf{T}^{(t)}\textbf{W}_T^{(t)}
\end{equation}

\emph{In this work we propose tunable parameters $\alpha$ and $\beta$ for weighting the degrees in message passing.
  This subtle modification can have important effects in performance for the link prediction task as demonstrated in the experimental section.} 

Similarly we update the \emph{target} encodings by aggregating the transformed and normalized \emph{source} encodings in the neighborhood:
\begin{equation}
  \textbf{T}^{(t+1)} \leftarrow \left({\tilde{\textbf{D}}^{-}}\right)^{-\alpha}
  \tilde{\textbf{A}}^{\top}
  \left({\tilde{\textbf{D}}^{+}}\right)^{-\beta}
  \textbf{S}^{(t)}\textbf{W}_S^{(t)}
\end{equation}
where $\textbf{W}_T^{(t)}$ and $\textbf{W}_S^{(t)}$ are the (learnable) linear tranformations for target and encoding encodings prior to their propagation. 
Defining $\hat{\textbf{A}} = \left({\tilde{\textbf{D}}^{+}}\right)^{-\beta} \tilde{\textbf{A}}  \left({\tilde{\textbf{D}}^{-}}\right)^{-\alpha}$, we can 
compactly express our proposed graph convolutional layer by the following pair of tranformations:
\begin{equation} \label{eq:neural-encodings}
  \begin{split}
  \textbf{S}^{(t+1)} & \leftarrow  \hat{\textbf{A}} \; \textbf{T}^{(t)} \; \textbf{W}_T^{(t)} \\
  \textbf{T}^{(t+1)} & \leftarrow  \hat{\textbf{A}}^{\top} \; \textbf{S}^{(t)} \; \textbf{W}_S^{(t)}.
  \end{split}
\end{equation}

\subsubsection{Reduction to 1-WL}
We consider the transformation of our directed, unweighted graph $G(V, E)$ to its \emph{bipartite representation} $G_{st}(V_s, V_t, E_{st})$
with  $V_s = V = [0, n)$, $V_t = [n, 2 \; n)$ and $E_{st} = \{\{i, j + n\} | (i, j) \in E\}$ \cite{bang2008digraphs}.
The bipartite graph $G_{st}$ is assumed undirected and it follows that for the set $\mathcal{N}_{st}(v)$ of the immediate neighbors of any of its nodes $v \in V(G_{st}) = V_s \cup V_t$:

 \begin{equation}\label{eq:bipartite}
 \mathcal{N}_{st}(v)=
\left\{
	\begin{array}{ll}
		\{j+n | (v, j) \in E\}  & \mbox{if } 0 \leq v < n \\
		\{i | (i, v-n) \in E\} & \mbox{if } n \leq v < 2 
	\end{array}
\right.  
\end{equation}

We also assume that the $2 n$ nodes of $G_{st}$ are initially colored with two colors in the set $\{s, t\}$ according to a coloring function function $l : V_s\cup V_t \rightarrow \Sigma$ with $l(v)=s$, if  $0 \leq v < n$ and  $l(v)=t$, if  $n \leq v < 2 \;n$. In other words $(G_{st}, l)$ is a labeled graph, so we can directly apply the standard 1-WL iterative algorithm for color refinement. In each iteration $ t\geq 0$, 1-WL computes a node coloring $c_{l}^{(t)} : V(G_{st}) \rightarrow \Sigma$ with an arbitrary codomain of colors $\Sigma$. We initialize this sequence of colorings as $c_l^{(0)} = l$ and at $t > 0$, node coloring is updated as:
\begin{equation*} \label{eq:wl}
  c_l^{(t)}(v) = \mathtt{HASH}\left( (c_l^{(t-1)}(v), \{\{ c_l^{(t-1)}(w)| w \in \mathcal{N}_{st}(v)\}\})\right)
  \end{equation*}
Termination of this standard 1-WL algorithm is guaranteed after at most $|V(G_{st})|=2 \;n $ iterations and reached when $|c_l^{(t)}| = |c_l^{(t-1)}|$ for some $k$.

Essentially, the bipartite representation maps a node $i \in V$ in the directed graph $G$ to a pair of nodes $(i \in V_s, i+n \in V_t)$ in the undirected, bipartite graph $G_{st}$: $i \in V_s$ represents $i \in V$ as a \emph{source} of directed edges in $G$ and $i+n \in V_t$ represents $i \in V$ as a \emph{target} of directed edges in $G$. In addition, identical color pair labels $(s, t)$, initially assigned to all nodes $i \in  V$ (in the directed graph $G$) are mapped to: (i) an initial label $s$ for node $i \in V_s$ and (ii) an initial label $t$ for node $i+n \in V_t$ (in the undirected, bipartite graph $G_{st}$). So in our refinement scheme for a node $i$ in $G$, it holds $(c_{l, s}^{(0)}(i), c_{l, t}^{(0)}(i)) = (s, t)$ while for its corresponding nodes $i$, $i+n$ in $G_{st}$, $c_l^{(0)}(i)=s$ and $c_l^{(0)}(i+n)=t$.

We can now prove that the color pairs computed at each step $k$ and for each node $i$ by our refinement scheme on a directed graph $G$ can equivalently be computed by standard 1-WL on the undirected, bipartite $G_{st}$ by pairing the colors of its nodes $i$ and $i+n$. In other words:

\begin{theorem}
\begin{equation}\label{eq:induction-step}
  c_{l, s}^{(t)}(i) = c_l^{(t)}(i), c_{l, t}^{(t)}(i) = c_l^{(t)}(i+n)
\end{equation}

for all layers $t\geq 0$ and $i \in [0, n)$.
\end{theorem}

\begin{proof}
We use induction. For $t=0$ this holds by our initialization convention: $c_{l, s}^{(0)}(i) = c_l^{(0)}(i) = s$ and  $c_{l, t}^{(0)}(i) = c_l^{(0)}(i+n) = t$.

Let us now assume that this also holds for some $t$ and find out for $t+1$.
Consider a node with index $i \in [0, n)$. Its neighbors $w = j + n \in \mathcal{N}_{st}(i)$ in $G_{st}$ map to the outlink-neighbors $j \in \mathcal{N}^{+}(i)$ in $G$ according to our bipartite representation (see Equation (\ref{eq:bipartite}). These neighbors $w=j+n \in \mathcal{N}_{st}(i)$ have color labels  $c_l^{(t)}(j+n)$ which according to (\ref{eq:induction-step}) are equal to $c_{l, t}^{(t)}(i)$: 
\begin{equation}\label{eq:induction-source}
  \begin{split}
    c_l^{(t+1)}(i) = & \mathtt{HASH}
    (
    (c_l^{(t)}(i), \{\{ c_l^{(t)}(w)| w \in \mathcal{N}_{st}(i)\}\})
    ) = \\
    & \mathtt{HASH}
    (
    (c_{l, s}^{(t)}(i), \{\{c_{l, t}^{(t)}(j)| j \in \mathcal{N}^{+}(i)\}\})
    )
  \end{split}
\end{equation}

Similarly for a node with index $i + n \in [n, 2 n)$, its neighbors $w = j \in \mathcal{N}_{st}(i)$ in $G_{st}$ map to the inlink-neighbors $j \in \mathcal{N}^{-}(i-n)$ in $G$ according to Equation \ref{eq:bipartite}. The node with index  $i + n$ has color label equal to $c_{l, t}^{(t)}$ and the inlink neighbors $j$, since $j \in [0, n)$ have color labels equal to  $c_{l, s}^{(t)}$ (see \ref{eq:induction-step}): 
\begin{equation}\label{eq:induction-target}
  \begin{split}
    c_l^{(t+1)}(i + n) = & \mathtt{HASH}
    (
    (c_l^{(t)}(i + 1), \{\{ c_l^{(t)}(w)| w \in \mathcal{N}_{st}(i)\}\})
    ) = \\
    & \mathtt{HASH}
    (
    (c_{l,t}^{(t)}(i), \{\{c_{l, s}^{(t)}(j)| j \in \mathcal{N}^{-}(i)\}\})
    )
  \end{split}
\end{equation}

According to our extension to 1-WL, the refinement we propose reads:
$c_{l, s}^{(t+1)}(i) = \mathtt{HASH}\left( (c_{l, s}^{(t)}(i), \{\{c_{l, t}^{(t)}(j)| j \in \mathcal{N}^{+}(i)\}\})\right)$ and
$c_{l, t}^{(t+1)}(i) = \mathtt{HASH}\left( (c_{l, t}^{(t)}(i), \{\{c_{l, s}^{(t)}(j)| j \in \mathcal{N}^{-}(i)\}\})\right)$.
By comparing with equations (\ref{eq:induction-source}) and (\ref{eq:induction-target}) it follows that
$c_{l, s}^{(t+1)} = c_l^{(t+1)}(i)$ and $c_{l, t}^{(t+1)} = c_l^{(t+1)}(i+n)$ which completes the proof.
\end{proof}

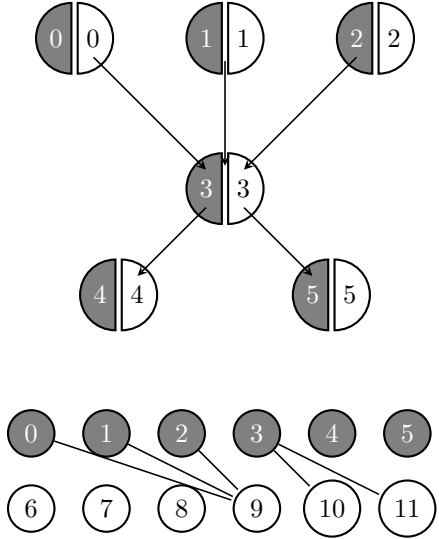
\begin{figure}
  \centering
  \begin{tikzpicture}[
    scale=0.2,
    > = stealth, 
    shorten > = 1pt, 
    auto,
    node distance = 2cm, 
    semithick 
    ]

  \tikzstyle{every state}=[
  draw = black,
  thick,
  fill = white,
  minimum size = 4mm
  ]

  \tikzset{
    semi/.style={
      semicircle,
      draw,
      minimum size=4mm
    }
  }

  \node (3) {};
  \node[state, semi, shape border rotate=90,  fill=gray, text=white, left=-1mm of 3] {$3$};
  \node[state, semi, shape border rotate=270, fill=white, text=black, right=-1mm of 3] {$3$};

  \node (1) [above of=3]{};
  \node[state, semi, shape border rotate=90,  fill=gray, text=white, left=-1mm of 1] {$1$};
  \node[state, semi, shape border rotate=270, fill=white, text=black, right=-1mm of 1] {$1$};
 
  \node (0) [left of=1]{};
  \node[state, semi, shape border rotate=90,  fill=gray, text=white, left=-1mm of 0] {$0$};
  \node[state, semi, shape border rotate=270, fill=white, text=black, right=-1mm of 0] {$0$};

  \node (2) [right of=1]{};
  \node[state, semi, shape border rotate=90,  fill=gray, text=white, left=-1mm of 2] {$2$};
  \node[state, semi, shape border rotate=270, fill=white, text=black, right=-1mm of 2] {$2$};

  \node (4) [below left of=3]{};
  \node[state, semi, shape border rotate=90,  fill=gray, text=white, left=-1mm of 4] {$4$};
  \node[state, semi, shape border rotate=270, fill=white, text=black, right=-1mm of 4] {$4$};

  \node (5) [below right of=3]{};
  \node[state, semi, shape border rotate=90,  fill=gray, text=white, left=-1mm of 5] {$5$};
  \node[state, semi, shape border rotate=270, fill=white, text=black, right=-1mm of 5] {$5$};
 
  \begin{scope}[> = stealth, semithick ,shorten >= 5pt, shorten <= 5pt]
  \path[->] (0) edge node {} (3);
  \path[->] (1) edge node {} (3);
  \path[->] (2) edge node {} (3);
  \path[->] (3) edge node {} (4);
  \path[->] (3) edge node {} (5);
  \end{scope}
\end{tikzpicture}

\vspace{1cm}

\begin{tikzpicture}[
  scale = 0.2,
  > = stealth, 
  shorten > = 1pt, 
  auto,
  node distance = 1cm, 
  semithick 
  ]
  
  \tikzstyle{every state}=[
  draw = black,
  thick,
  fill = white,
  minimum size = 4mm
  ]

  \node[state, fill=gray, text=white] (0) {$0$};
  \node[state, fill=gray, text=white] (1) [right of=0] {$1$};
  \node[state, fill=gray, text=white] (2) [right of=1] {$2$};
  \node[state, fill=gray, text=white] (3) [right of=2] {$3$};
  \node[state, fill=gray, text=white] (4) [right of=3] {$4$};
  \node[state, fill=gray, text=white] (5) [right of=4] {$5$};

  \node[state] (6)  [below of=0] {$6$};
  \node[state] (7)  [below of=1] {$7$};
  \node[state] (8)  [below of=2] {$8$};
  \node[state] (9)  [below of=3] {$9$};
  \node[state] (10) [below of=4] {$10$};
  \node[state] (11) [below of=5] {$11$};

  \path[-] (0) edge node {} (9);
  \path[-] (1) edge node {} (9);
  \path[-] (2) edge node {} (9);
  \path[-] (3) edge node {} (10);
  \path[-] (3) edge node {} (11);
\end{tikzpicture}
\caption{\footnotesize An example directed graph $G$ with $n=6$ nodes (top) and its bipartite representation $G_{st}$ (bottom). Both graphs are colored with each node in $G$ carrying a pair of source and target colors - respectively in left and right semidiscs. The example node $3$ in $G$ maps to nodes $3$ and $9$ in $G_{st}$.}  
\label{fig:bipartite-representation}
\end{figure}

\subsubsection{Connecting 1-WL to the encoder layers in DiGAE}
We can encode the initial labels $c_l^{(0)}(v)$ of the nodes in our bipartite representation $G_{st}$ by considering two arbitrary, nonequal, vectors $\textbf{s}$, $\textbf{t}$ in $\mathbb{R}^{e \times 1}$ and setting $f^{(0)}(v) = \textbf{s}^{\top}$  for nodes $v \in [0, n)$  and $f^{(0)}(v) = \textbf{t}^{\top}$ for nodes $v \in [n, 2 n)$;
$e$ can be as small as $1$ (i.e. for one-hot encoding). These encodings are consistent with the initial labels (i.e. they are different for nodes with different initial colors). Given the consistent encodings and our reduction of our color refinement to 1-WL for $G_{st}$, two results readily apply, from theorems 1 and 2 in \cite{morris2019weisfeiler}. These results readily establish the connection of our 1-WL to 1-GNNs (\ref{eq:1-GNN}).

Alternatively, we can encode the initial labels $c_l^{(0)}(v)$ of the nodes in our bipartite representation $G_{st}$ by vectors $f^{(0)}(v)$ as follows. For nodes $v \in [0, n)$ we assume the  encoding  $f^{(0)}(v) = [\textbf{s}^{\top}, \textbf{0}]$  and for nodes $v \in [n, 2 n)$  we set  $f^{(0)}(v) = [\textbf{0}, \textbf{t}^{\top}]$ where  $\textbf{0} \in \mathbb{R}^{1 \times e}$ denotes the zero row vector. Our encodings are in $\mathbb{R}^{1 \times 2 e}$ and they are also consistent with our initial labels. In this encoding scheme, the first $e$ elements (last $e$ elements) of the initial feature vector $f^{(0)}(v)$ are non-vanishing for nodes in $G_{st}$ capturing the source (target) role of nodes in $G$.

We can ensure this property is preserved for feature vectors $f^{(t)}(v), t > 0$ under the iteration in (\ref{eq:1-GNN}) by additionally selecting block-diagonal matrices for $W_1^{(t)}$ and block-antidiagonal matrices for $W_2^{(t)}$:
\begin{equation}
  W_1^{(t)} =
  \begin{pmatrix}
    W_{1S}^{(t)} &  O \\
    O & W_{1T}^{(t)}
  \end{pmatrix},
  W_2^{(t)} =
  \begin{pmatrix}
    O & W_{2S}^{(t)}\\
    W_{2T}^{(t)} &  O 
  \end{pmatrix}
\end{equation}

where $W_{1S}^{(t)}, W_{1T}^{(t)}, W_{2S}^{(t)}, W_{2T}^{(t)} \in \mathbb{R}^{e \times e}$ and $O=O^{e \times e}$ is the zero matrix.

This is straightforward to verify. For a node in $i \in [0, n)$ for which $f^{(t-1)}(i) = [\textbf{s}^{(t-1)}_i, \textbf{0}]$ its neighbors $j \in [n, 2 n)$ will have encodings of the form $f^{(t-1)}(v) = [\textbf{0}, \textbf{t}^{(t-1)}_j]$. Substitution in Equation (\ref{eq:1-GNN}) yields:
\begin{equation*}
  \begin{split}
    f^{(t)}(i) = 
    [\sigma(
    {\textbf{s}^{(t-1)}_i}^{\top} \cdot W_{1S}^{(t)} +
    \sum_{j \in N(i)} {\textbf{t}^{(t-1)}_j}^{\top} \cdot W_{2T}^{(t)}
    ), \textbf{0}^{1 \times e}]
  \end{split}
\end{equation*}
so $f^{(t)}(i)$ still has only its \emph{first} $e$ elements non-vanishing, the same as  $f^{(t-1)}(i)$.
Similarly for $i \in [n, 2 n)$ we get: $f^{(t)}(i) = 
    [\sigma(\textbf{0}^{1 \times e},      {\textbf{t}^{(t-1)}_i}^{\top} \cdot W_{1T}^{(t)} +
   \sum_{j \in N(i)} {\textbf{s}^{(t-1)}_j}^{\top} \cdot W_{2S}^{(t)}     )]$, 
so $f^{(t)}(i)$ still has only its \emph{last} $e$ elements non-vanishing, the same as  $f^{(t-1)}(i)$.

   The motivation behind the 
   selection of block-structured parameter matrices $W_1^{(t)}$ and $W_2^{(t)}$, is for allowing the extra flexibility of learning different sets of parameters for nodes with source and target roles (in undirected $G_{st}$). By removing self-links and setting $W_{1S}^{(t)}= W_{1T}^{(t)}=O^{e\times e}$, these vectors $f^{(t)}(i)$ can be seen to map directly to rows of matrices $\textbf{S}^{(t)}$ and $\textbf{T}^{(t)}$ in Equation  (\ref{eq:neural-encodings}). Alternatively, we can consider general (non-structured) $W_1^{(t)}$ and $W_2^{(t)}$ in our 1-GNN which corresponds to sharing learnt parameters for neural source and target encodings.

\subsection{Directed Graph AutoEncoder (DiGAE) model}

In analogy to the Graph AutoEncoder (GAE) model in \cite{kipf2016variational}, we can define its directed variant (DiGAE) by stacking two directed graph convolutional layers connected with \texttt{ReLU} nonlinearity for its \emph{encoder} and a \texttt{sigmoid} applied to the inner product of the source and target encodings for its \emph{decoder}. In particular, the \emph{encoder} reads
\begin{equation}
\begin{split}
\textbf{Z}_S &= \hat{\textbf{A}} \texttt{ReLU}(\hat{\textbf{A}}^{\top} \; \textbf{S}^{(0)} \; \textbf{W}_S^{(0)}) \textbf{W}_T^{(1)} \\
\textbf{Z}_T &= \hat{\textbf{A}}^{\top} \texttt{ReLU}(\hat{\textbf{A}} \; \textbf{T}^{(0)} \; \textbf{W}_T^{(0)}) \textbf{W}_S^{(1)}
\end{split}
\end{equation}

while the decoder for computing the reconstructed adjacency matrix $\bar{\textbf{A}}$ is

\begin{equation}
  \bar{\textbf{A}} = \sigma(\textbf{Z}_S \; \textbf{Z}_T^{\top})
  \label{eq:directed-decoder}
\end{equation}

In the sequel, we experiment with the single-layer DiGAE model, referred to as \emph{DiGAE-1L}. This autoencoder has the same decoder as DiGAE and its encoder implements the pair of transformations
$\textbf{Z}_S = \textbf{S}^{(1)} = \hat{\textbf{A}} \textbf{T}^{(0)} \textbf{W}_T^{(0)}$ and
$\textbf{Z}_T = \textbf{T}^{(1)} = \hat{\textbf{A}}^{\top} \textbf{S}^{(0)} \textbf{W}_S^{(0)}$.

\section{Related work}

Ma et al \cite{ma2019spectral} formally extend convolution to directed graphs by defining a normalized, symmetric directed Laplacian and leveraging the Perron vector of the induced transition probability matrix for weighing messaged passing.
In \cite{tong2020digraph}, Tong et al approximate the digraph Laplacian by means of Personalized PageRank, and this relaxed definition offers performance benefits and the ability to process directed graphs that are not necessarily strongly connected. 
The authors of Fast Directed GCN in \cite{li2020scalable} also approximate the digraph Laplacian, by assuming the trivial Perron vector for regular graphs.
In \cite{monti2018motifnet}, a polynomial of selected motif Laplacians is used to filter node representations. Most notably, all aforementioned spectral-based extensions to convolution in directed graphs, produce single-vector representations, so the dual nature of a node as both a source and target of directed edges is not captured; we do not require the separate computation of the Perron vector and a scheme for weighing messages is automatically learnt. For convolution in directed knowledge graphs, we refer to \cite{kampffmeyer2019rethinking,schlichtkrull2018modeling}.

In \cite{tong2020directed}, first and second order proximity kernels for directed graphs are combined to produce single-vector representations for graph nodes. Second order proximity  kernels normalize the products of $\textbf{A}$ and its transpose, and produce dual-vector intermediate representations, similarly to the work proposed in this paper. However, normalization of the products is graph-agnostic, the intermediate sparse matrix products they employ are computationally costly with potentiallly dense matrix outputs, the first order proximity 
kernels is symmetric.
High-Order Proximity preserved Embeddings (HOPE) in \cite{ou2016asymmetric} rely on matrix factorization (SVD) of a higher order proximity matrix while Asymmetric Proximity Preserving (APP) embeddings in \cite{zhou2017scalable} rely on random walks with restart. 
These encoders are not GCN-based which is the focus in our work.
Dual-vector representation can be enforced artificially: in \cite{salha2019gravity}, source/target GAE and VGAE models are based on graph auto-encoders for undirected graphs from the seminal work of Kipf and Welling \cite{kipf2016variational}, where the single-vector representation (of even size) is assumed to be the concatenation of two-identically sized parts
The gravity-inspired directed GCN architecture in \cite{salha2019gravity} on the other hand produces single vector encodings using fixed normalization for the messages and embeds asymmetry in only one of its entries by fusing the importance of the target node and the distance of the node encodings at the ends of the directed edge.

The directed graph can be of a special kind. Message passing in GCNs for Directed Acyclic Graphs (DAGs)
is explored in \cite{thost2021directed}. Or the directed edges can carry labels of different types in which case the undirected GCN can be extended to include terms denoting the additional aggregation of edge features \cite{jaume2019edgnn}, in close analogy to the extension of WL test to a directed graph with edge labels \cite{orsini2015graph, grohe2017color}. In another interesting view, the directed graph can first be converted to an undirected bipartite graph as in \cite{zhou2005semi} driven by the source-target node duality.

\section{Experiments}

In this section we demonstrate the performance of the proposed approach on the {\em directed} link prediction task 
associated with two different datasets: (a) namely CoraML (2,995 nodes, 8,416 edges, 2,879 features), and (b) CiteSeer (3,312 nodes, 4,715 edges, 3,703 features).
The CoraML dataset 
contains machine learning publications grouped into seven  classes.  The  CiteSeer  dataset  
contains  scientific papers  grouped  into  six  classes.  Each  paper  in  CoraML  and CiteSeer 
is  represented  by  a  one-hot  vector  indicating  the presence or absence of a word from a dictionary.

We employ grid search for hyperparameter tuning: learning rate $\eta \in \{0.005, 0.01\}$, hidden layer dimension $d \in \{32, 64\}$ with $d/2$ for the latent space dimension, $(\alpha, \beta) \in \{0.0, 0.2, 0.4, 0.6, 0.8\}^{2}$ for DiGAE models and parameter $\lambda \in \{0.1, 1.0, 10.0\}$ for Gravity GAE. For the final models we select hyperparameter values that maximize mean AUC computed on the validation set. In all cases, models are trained for 
$200$ epochs, using Adam optimizer, without dropout, performing full-batch gradient descent.

We use Python and especially the PyTorch library and PyTorch Geometric \cite{Fey/Lenssen/2019}, which is a geometric deep learning extension library. We ran experiments for all models (ours and baselines) on a system equipped with an Intel(R) Core(TM) i7-8850H CPU @2.60GHz (6 cores/12 threads) and 32 GB of DDR4 memory @2400 MHz.

We consider the following directed link prediction task 
adapting the description in \cite{kipf2016variational} to digraphs. 
\paragraph{\textbf{Task: Directed link prediction}}
We randomly remove $15\%$ of the directed edges from the graph and train models on the remaining edge set. 
Two thirds of the removed edges (i.e. $10\%$ of all input graph edges) are used as actual edges for testing, one third of them (i.e. $5\%$ of all input graph edges) as actual edges for validation. These test and validation sets also include the same number of fake directed edges (negative samples), as their actual ones: negative samples are generated by randomly connecting pairs of nodes which are not wired in the input graph.
Validation sets are only used for hyperparameter tuning.

\subsection{Results and Discussion}
We compare the performance of our DiGAE models to Standard GAE in \cite{kipf2016variational}, Source/Target GAE and Gravity GAE both in \cite{salha2019gravity}. Similarly to our models, these baselines are GCN-based.
We use the TensorFlow implementations of baseline models from the 
authors of \cite{salha2019gravity}.\footnote{\url{https://github.com/deezer/gravity_graph_autoencoders}}

We run a series of twenty experiments for each graph and selected model combination, and report the mean and standard deviation of the AUC and AP metrics, as well as their respective timings. Random train and test graph dataset splits are used for each such experiment and metrics are averaged over the set, to account for their sensitivity to the choice of the split, as empirically observed in \cite{shchur2018pitfalls}.

Results
with node features supplied in all cases (feature-based configurations), are summarized in Table \ref{tab:task_1_features}.
We mark in bold the largest entry (or largest entries in the case of their overlapping intervals).

Mean AUC values for DiGAE-1L in particular consistently outperform other baselines for the citation graphs and the margin can be significant: in  CiteSeer this is of the order of $6\%$ for Gravity GAE and up to $18\%$ compared to Standard GAE. For mean AP, margins are similarly in the $4\%$ to $19\%$ range, or up to $3\%$ for CoraML. Varying $(\alpha, \beta)$ pair values has significant impact on the reconstruction metrics. As an example, the mean AUC can be as low as $73.73\%$ and $86.10\%$ for DiGAE-1L  respectively for CoraML and CiteSeer for the selected $\eta$ and $d$ but for suboptimal $(\alpha, \beta)$ within the search grid. In the Appendix we include metrics tables for the full $(\alpha, \beta)$ grid collected during hyperparameter tuning.
In experiments with citation graphs, DiGAE-1L is markedly faster by factors in the range $\times 5$ to $\times 15$. This is partly due to the different implementations, the fact that this is a single-layer only and for CiteSeer in particular to the smaller size for the output vector from hyperparameter tuning ($16$ vs $32$ for a hidden encoding of $64$ entries for baselines). Gravity GAE is the slowest, mainly due to the complexity of its decoder (pairwise distance computation).

The authors in \cite{salha2019gravity} originally used the baselines with one-hot encoding of the nodes (feature-less configurations).
For citation networks, node features capture similarity which is symmetric and this could conceptually hinder the identification of directionality in predicted links as in our task, which is an inherently asymmetric relation. For this reason, we also tested feature-less configurations for the baselines and interestingly got comparable results to the feature-based case (Table \ref{tab:task_1_gravity}): our DiGAE-1L models from (Table \ref{tab:task_1_features}) are top performers for both CoraML and CiteSeer datasets.
Experiments with datasets from the WebKB collection and Pubmed are included in the Appendix.
\begin{table}[ht]
\begin{center}
  \caption{\footnotesize General directed link prediction for the \emph{citation} graphs (feature-based configurations).}
 \label{tab:task_1_features}
\begin{tabular}{lllll}
\toprule
 Dataset &                Model &            AUC &             AP &            Time (secs) \\
\midrule
  CoraML &         DiGAE (ours) & 88.10 +/- 1.70 & 89.83 +/- 1.39 &   9.56 +/- 0.48 \\
  CoraML & DiGAE 1-Layer (ours) & \textbf{94.09 +/- 0.66} & \textbf{94.10 +/- 0.77} &   7.64 +/- 0.21 \\
  CoraML &          Gravity GAE & 92.35 +/- 0.57 & \textbf{94.17 +/- 0.53} &  51.56 +/- 0.13 \\
  CoraML &    Source/Target GAE & 91.66 +/- 0.52 & 92.60 +/- 0.47 &  36.88 +/- 0.20 \\
  CoraML &         Standard GAE & 89.54 +/- 1.14 & 91.40 +/- 1.11 &  36.55 +/- 0.21 \\
\midrule
CiteSeer &         DiGAE (ours) & 92.05 +/- 1.06 & 92.29 +/- 0.97 &  11.10 +/- 0.22 \\
CiteSeer & DiGAE 1-Layer (ours) & \textbf{92.76 +/- 0.87} & \textbf{92.57 +/- 1.08} &   4.12 +/- 0.11 \\
CiteSeer &          Gravity GAE & 86.79 +/- 0.98 & 88.60 +/- 1.05 & 106.44 +/- 0.16 \\
CiteSeer &    Source/Target GAE & 82.90 +/- 1.79 & 84.70 +/- 1.47 &  73.46 +/- 0.21 \\
CiteSeer &         Standard GAE & 74.35 +/- 1.77 & 80.96 +/- 1.26 &  73.84 +/- 0.48 \\
\bottomrule
\end{tabular}
\end{center}
\end{table}

\begin{table}[ht]
  \begin{center}
  \caption{\footnotesize General directed link prediction for the \emph{citation} graphs (feature-less configurations).}
\label{tab:task_1_gravity}
  \begin{tabular}{llll}
\toprule
 Dataset &             Model &            AUC &             AP \\
\midrule
  CoraML &       Gravity GAE & 91.56 +/- 0.71 & 93.62 +/- 0.59 \\
  CoraML & Source/Target GAE & 91.85 +/- 0.58 & 92.91 +/- 0.64 \\
  CoraML &      Standard GAE & 89.42 +/- 0.84 & 92.09 +/- 0.64 \\
  \midrule
CiteSeer &       Gravity GAE & 87.05 +/- 0.99 & 88.88 +/- 0.96 \\
CiteSeer & Source/Target GAE & 82.39 +/- 1.56 & 84.38 +/- 1.57 \\
CiteSeer &      Standard GAE & 73.81 +/- 1.37 & 80.52 +/- 1.13 \\
\bottomrule
\end{tabular}
\end{center}
\end{table}
\paragraph{\textbf{Truncated SVD baseline}}
\begin{figure}[ht]
  \centering
  \includegraphics[width=0.7\textwidth]{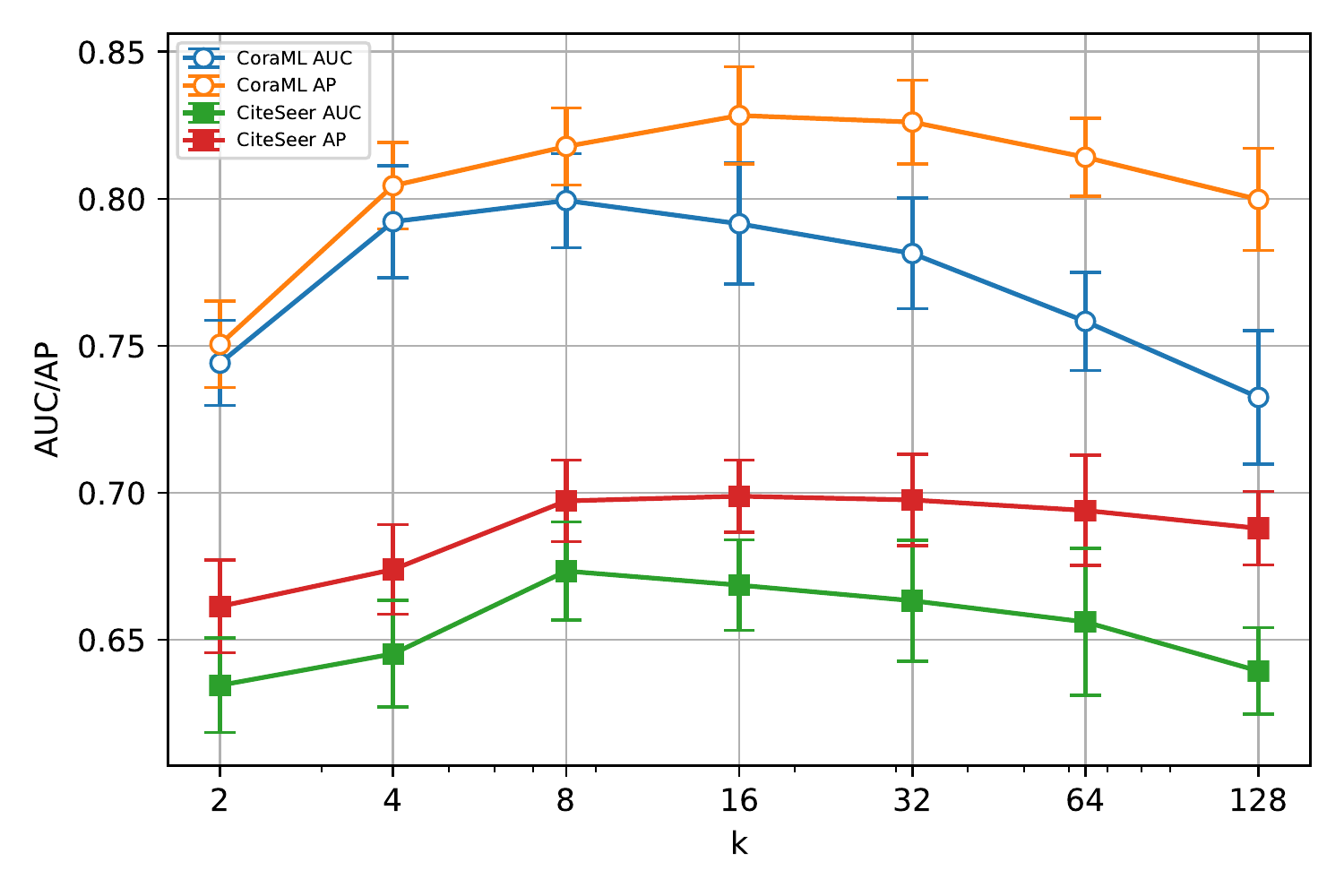}
  \caption{\footnotesize Truncated SVD based directed link prediction: AUC and AP metrics for CoraML and CiteSeer.}
  \label{fig:truncated-svd-baseline}
\end{figure}

\begin{figure*}[ht]
  \centering
  \includegraphics[width=0.30\textwidth]{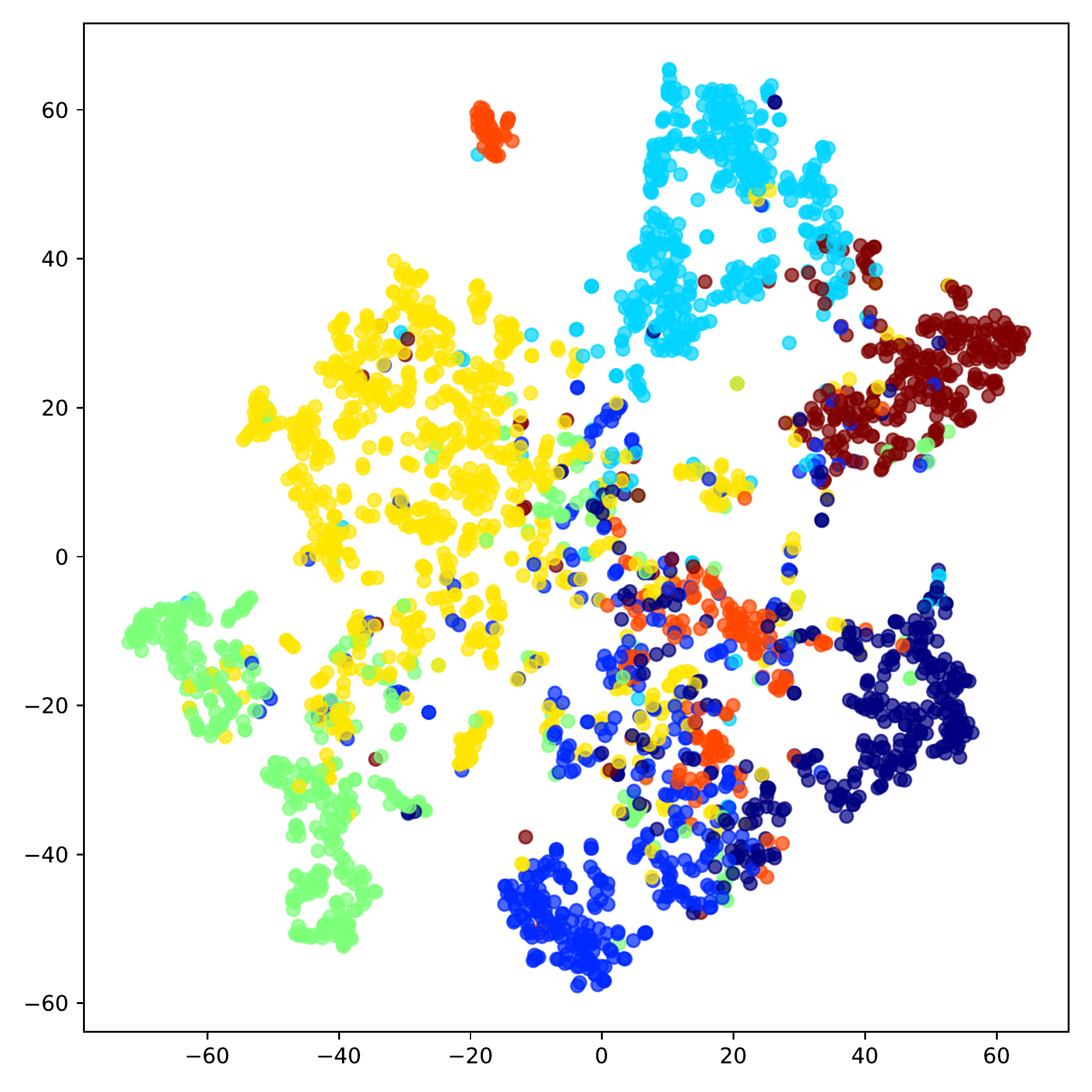}
  \includegraphics[width=0.30\textwidth]{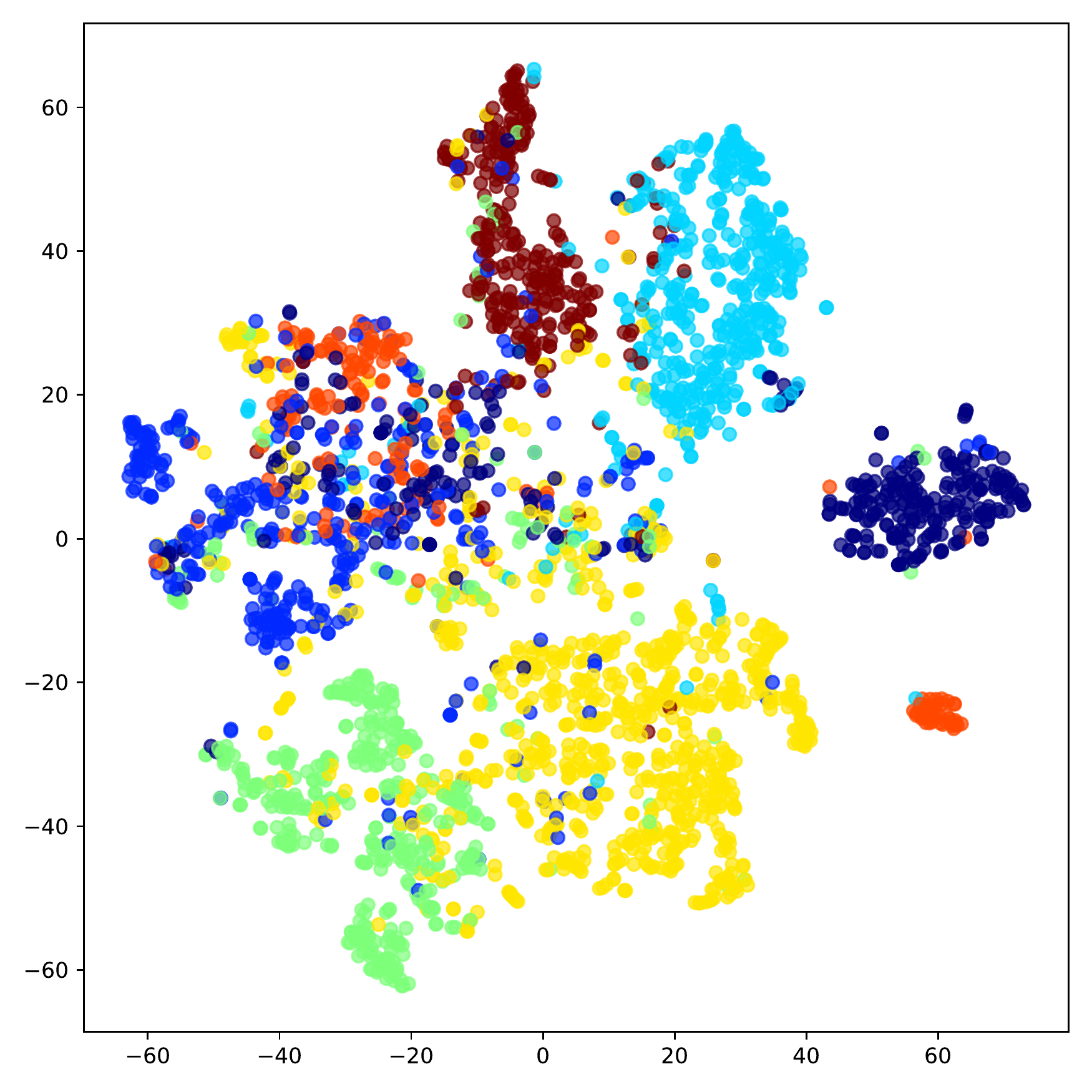}
  \includegraphics[width=0.30\textwidth]{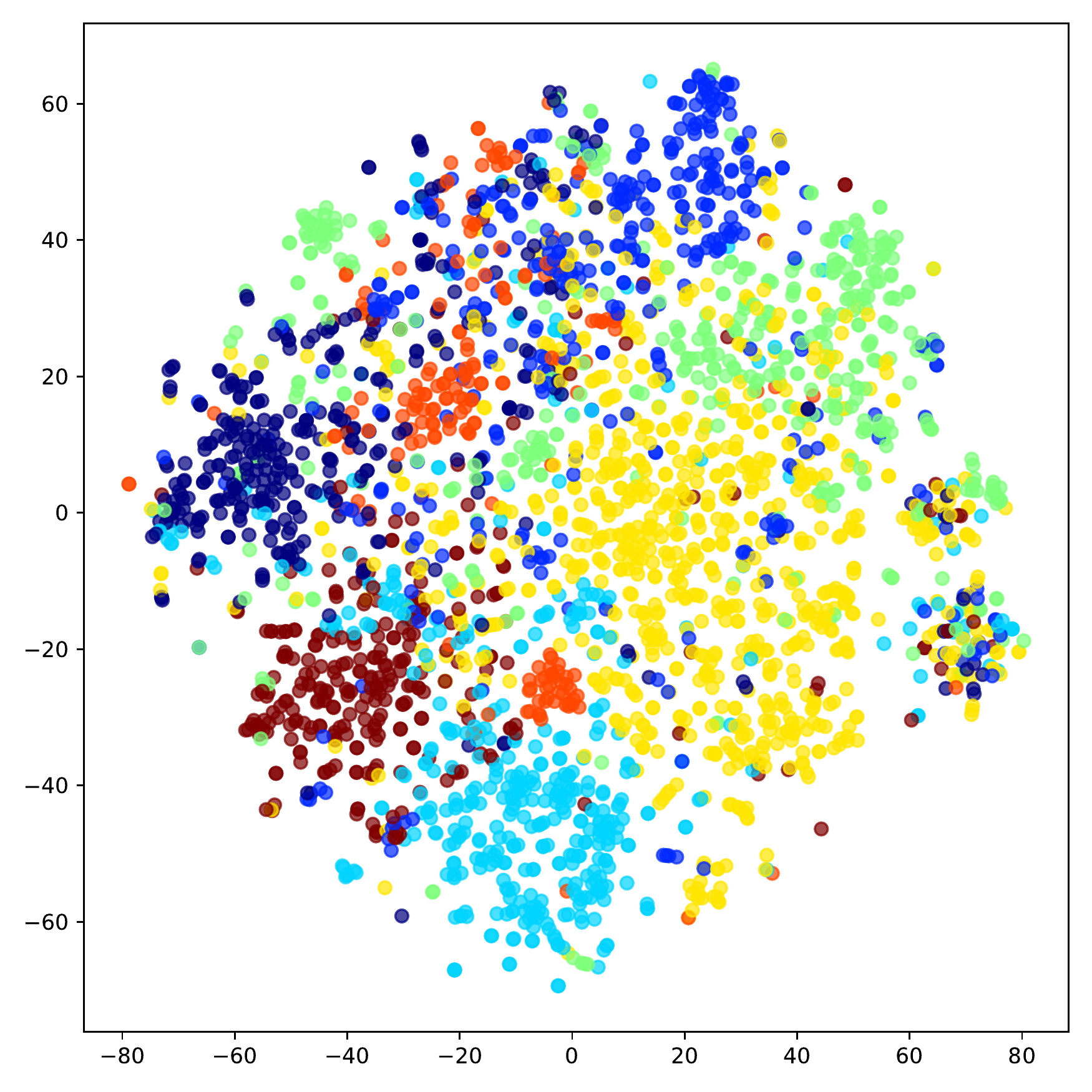}
  \caption{\footnotesize Left and center: t-SNE embeddings on source and target representations 
  computed with our DiGAE-1L for CoraML. Right: embeddings of orginal feature 
  vectors (colors indicate classes).}
  \label{fig:tsne-spaces}
\end{figure*}

We also explore the performance of an approach based on truncated SVD for our directed link prediction task.
For a  given train/validation/test split, we use the directed edges that are available to use for training for building
a \emph{partial} adjacency matrix $A_{p}$
of the underlying directed graph $G(V, E)$ and compute its truncated SVD,
retaining the singular triplets corresponding to its largest $k$ singular values,
$\textbf{A}_{p} \approx \textbf{U}_{k} \mathbf{\Sigma}_k \textbf{V}_{k}^\top$.
Then we consider the matrices for the source and target encodings of nodes
$\textbf{Z}_{S}= \textbf{U}_{k} \mathbf{\Sigma}_k^{1/2}$ and
$\textbf{Z}_{T} = \textbf{V}_{k} \mathbf{\Sigma}_k^{1/2}$
and use the asymmetric decoder from Equation (\ref{eq:directed-decoder}) for predicting directed links.
We use the same evaluation pipeline as in our GNN experiments and compute respective AUC and AP performance metrics for the graph reconstruction.
For truncated SVD we experimented with implementations based on \texttt{ARPACK} \cite{books/daglib/0000896} and randomized SVD \cite{halko2011finding} for
$k \in \{2^{i} | i=1, 2, \ldots, 7\}$ and we got very similar metrics for our input graphs (please refer to the Appendix for metrics tables). We repeated for $20$ random graph splits for each graph, $k$ and implementation combination. Figure \ref{fig:truncated-svd-baseline} summarizes the results for the \emph{largest} mean AUC and AP value for each graph, $k$ combination. We observe that for 16-dimensional encoding vectors we get best results in reconstruction in terms of mean AP: $82.83\%$ for CoraML and $69.88\%$ for CiteSeer. These are significantly lower than $94.10\%$ for CoraML and $92.57\%$ for CiteSeer that the ``neural'' source and target encodings our DiGAE-1L produces. SVD encodes only connectivity while DiGAE models, being GNNs, leverage both connectivity \emph{and} node features.
The truncated SVD baseline is essentially the HOPE idea \cite{ou2016asymmetric} with the proximity matrix being the adjacency matrix \textbf{A}.
For comparison and for the standard Katz proximity in HOPE, ${(\textbf{I} - \beta \textbf{A})}^{-1} \beta \textbf{A}$,
we computed \emph{largest} mean AP values in reconstruction over $k \in \{2^{i}|i=1,2,\ldots, 7\}$, $\beta=0.02$, applying $20$ random graph splits for each $k$. We get $83.87\%$ for CoraML and $67.29\%$ for CiteSeer: comparable to truncated SVD baseline results and significantly lower than the respective DiGAE-1L metrics.
We include tables with detailed results in the Appendix.  

\paragraph{\textbf{Clustering source and target encodings}} We used t-Distributed Stochastic Neighbor Embedding (t-SNE) \cite{van2008visualizing} for reducing the dimension of source and target representations of nodes produced by DiGAE in order to visualize them. Figure \ref{fig:tsne-spaces} illustrates a crisp separation of points with different class labels and this is true for both source and target vectors. This separation is not present in their feature space. This implies that DiGAE embeddings are promising inputs for clustering purposes. The fact that the clusters can be independently identified in both the source and target space opens up the possibility for exploring sets of points that share cluster identity in both spaces, as the core ones for a label. 
\paragraph{\textbf{Vector hub and authority scores}} Source and target representations are expected to be the (vector) surrogates respectively of (scalar) hub and authority scores. Table \ref{tab:correlation} confirms this intuition. A source vector of particularly large magnitude for a node $a$ will be more probable to yield large inner products with target vectors of other nodes $b$ in the decoder. This means that a directed  edge $a \mapsto b$ will be likely to appear, which will increase $a$'s outdegree. Similarly, a large magnitude for a node's target vector encourages other nodes to connect to it and increase its indegree. In turn, outdegrees are known to correlate to hub scores and indegrees to authority and PageRank scores.  
\begin{table}[ht]
\centering
\caption{\footnotesize Pearson correlation coefficients between the magnitudes of source and target vector 
encodings, and a collection of centrality scores (Hub/Authority and PageRank) and degrees 
(in/out), for all nodes in CoraML. Encodings were computed with DiGAE-1L.}
\label{tab:correlation}
\begin{tabular}{lrr}
\hline
           &   source magnitude &   target magnitude \\
\hline
 hub       &            \textbf{0.37} &            0.06 \\
 outdegree &            \textbf{0.82} &            0.12 \\
 authority &            0.05 &            \textbf{0.41} \\
 pagerank  &           -0.01 &            \textbf{0.48} \\
 indegree  &            0.08 &            \textbf{0.77} \\
\hline
\end{tabular}
\end{table}
\section{Conclusions and future work}
In this paper we present DiGAE, a new class of directed graph autoencoders that computes a pair of vector representations for each node. It exploits the asymmetry in input and output node degrees and further skews this by allowing exponents in scaling the features to enter as parameters. DiGAE outperforms state-of-the art GCN-based graph autoencoders on the directed link prediction task and can be an order of magnitude times faster in learning representations for CoraML and CiteSeer datasets. In future work, we plan to explore encoders that integrate scaling decisions that are \emph{local} to each node.

\newpage
\clearpage
\section{\huge{Appendix}} \label{appendix}
\rule{\textwidth}{2pt}\\\vspace{0.5cm}

We include the detailed algorithm for WL and the message passing view of a simplified GCN layer.
We then describe our coloring of pairs of node labels in directed graphs and its associated message passing scheme.

In order to emphasize the impact of $\alpha$ and $\beta$ hyperparameters on performance metrics for CoarML and CiteSeer datasets, we visualize AUC mean values, collected during hyperparameter tuning,
for the \emph{selected} model's $\eta$ and $d$ values. We also provide the detailed performance metrics table for the truncated SVD baseline for both standard and randomized SVD implementations
on which Figure 2 in the main text is based and the standard Katz proximity in HOPE is similarly explored next. An application of all models on small Web graphs from the WebKB collection follows.
We conclude with experiments over a larger dataset, Pubmed, in which we confirm the performance benefits of our DiGAE models and the optimization opportunities that parameters $\alpha$, $\beta$ offer, followed
by a short note connecting them to the spectral modifications they effect.

\subsection{Weisfeiler-Leman (WL) algorithm}
\label{sec:weisfeiler-leman-wl}
Algorithm \ref{alg:1-WL} provides a summary of the Weisfeiler-Leman (WL) algorithm.
\begin{algorithm}
  \begin{algorithmic}
    \State {\textbf{Input:}} Neighbor lists $\mathcal{N}(i)$, for all nodes $i \in [0, n)$
    \State {Ensure neighbor lists are lexicographically sorted: $\mathcal{N}(i) \leftarrow \texttt{sort}(\mathcal{N}(i))$, $\forall i \in [0, n)$}
  
    \State{Initialize: $k\leftarrow 0$; $l_i^{(k)} := \texttt{c}$ (same label),
      $h_i^{(k)} := \texttt{hash}(i)$, $\forall i \in [0, n)$;  $L^{(0)} := \{l_i^{(0)}: i \in [0, n)\}$}
    \While{True}
    \State{Hash the tuple of (sorted) neighbors' labels: $h_i^{(k+1)} := \texttt{hash}(<l_j^{(k)}: j \in \mathcal{N}(i)>)$}
    \State{List of unique hashes $\mathcal{H}$: $\mathcal{H} \leftarrow \texttt{sort}$ ([\{$h_i^{(k+1)}: i \in [0, n)$\}])}
    \State{Function $f$ maps items in $\mathcal{H}$ to labels: $f(\mathcal{H}_j) := j,  \forall j \in [0, size(\mathcal{H}))$}
    \State{Update labels: $l_i^{(k+1)} := f(h_i^{(k+1)}))$,  $\forall i \in [0, n)$}
    \State{Update set of labels: $L^{(k+1)} := \{l_i^{(k+1)}: i \in [0, n)\}$}
    \If{$L^{(k+1)} = L^{(k)}$}
    \State{\textbf{break}}
    \Else
    \State{$k \leftarrow k + 1$}
    \EndIf{}
    \EndWhile
    \State {\textbf{Output:}} $L^{(k)}$
\end{algorithmic}
\caption{Weisfeiler-Leman (WL) algorithm} \label{alg:1-WL}
\end{algorithm}

\subsection{GCNs as message passing systems}
\label{sec:gcns-as-message}

Algorithm \ref{alg:mp-GCN} describes the steps each node $i$ performs to update its 
current encoding $\textbf{x}_i$ given the set of its neighbors $\mathcal{N}(i)$ and 
a transformation matrix $\textbf{W}$ (which is identical for all nodes).

\begin{algorithm}
  \begin{algorithmic}
    \State {\textbf{Input:}} $\textbf{x}_i$, $\mathcal{N}(i)$, $\textbf{W}$
    \State{Set:  $\tilde{\mathcal{N}}(i) := \mathcal{N}(i) \cup \{i\}$ and $\texttt{deg}(i) := | \tilde{\mathcal{N}}(i) |$}
    \State{Transform: $\textbf{x}_i \leftarrow \textbf{W}^{\top} \textbf{x}_i$ }
    \State{Scale: $\textbf{x}_i \leftarrow \frac{1}{\sqrt{\texttt{deg}(i)}} \; \textbf{x}_i$}
    \State{Send: $\texttt{send}(\textbf{x}_i, j)$ \emph{to all nodes} $j \in \tilde{\mathcal{N}}(i)$}
    \State{Receive:  $\textbf{x}_j \leftarrow \texttt{receive}(j)$ \emph{from all nodes} $j \in \tilde{\mathcal{N}}(i)$}
    \State{Sum: $\textbf{x}_i \leftarrow \sum_{j \in \tilde{\mathcal{N}}(i)}\textbf{x}_j$}
    \State{Scale: $\textbf{x}_i \leftarrow \frac{1}{\sqrt{\texttt{deg}(i)}} \; \textbf{x}_i$}
    \State {\textbf{Output:}} $\textbf{x}_i$
\end{algorithmic}
\caption{Updating the enconding of node $i$ in a GCN layer} \label{alg:mp-GCN}
\end{algorithm}

\subsection{Coloring of pairs of node labels in directed graphs}
\label{sec:coloring-pairs-node}
Algorithm \ref{alg:d-WL} provides a summary of our extension of 1-WL for directed graphs.
\begin{algorithm}
  \begin{algorithmic}
    \State {\textbf{Input:}} Neighbor lists for outgoing and incoming edges, $\mathcal{N}^{+}(i)$ and  $\mathcal{N}^{-}(i)$ respectively, for all nodes $i \in [0, n)$
    \State {Ensure neighbor lists are lexicographically sorted: $\mathcal{N}^{+}(i) \leftarrow \texttt{sort}(\mathcal{N}^{+}(i))$,
      $\mathcal{N}^{-}(i) \leftarrow \texttt{sort}(\mathcal{N}^{-}(i))$,
      $\forall i \in [0, n)$}
    
    \State{Initialize: $t\leftarrow 0$;
      $c_{l, s}^{(t)} := \texttt{s}, c_{l,t}^{(t)} := \texttt{t}$;
      $h_{s}^{(t)}(i) = \texttt{HASH}(\texttt{s}), h_{t}^{(t)}(i) := \texttt{HASH}(\texttt{t})$, $\forall i \in [0, n)$;
      $L_{s}^{(0)} = L_{t}^{(0)} := \{\texttt{s}, \texttt{t}\}$}
    \While{True}
    \State{Hash \emph{source} label, \emph{target} labels of pointed nodes:
    $h_{s}^{(t+1)}(i) := \texttt{HASH}(c_{l, s}^{(t)}(i), <c_{l, t}^{(t)}(j): j \in \mathcal{N}^{+}(i)>)$}
  \State{Hash node \emph{target} label, \emph{source} labels of pointing nodes:
    $h_{t}^{(t+1)}(i) := \texttt{HASH}(c_{l, t}^{(t)}(i), <c_{l, s}^{(t)}(j): j \in \mathcal{N}^{-}(i)>)$}
    \State{List of unique hashes for nodes with \emph{source} role: $\mathcal{H}_{s}$:
      $\mathcal{H}_{s} \leftarrow \texttt{sort}$ ([\{$h_{s}^{(t+1)}(i): i \in [0, n)$\}])}
    \State{List of unique hashes for nodes with \emph{target} role: $\mathcal{H}_{t}$:
      $\mathcal{H}_{t} \leftarrow \texttt{sort}$ ([\{$h_{t}^{(t+1)}(i): i \in [0, n)$\}])}
    \State{Function $f_s$ maps items in $\mathcal{H}_s$ to labels: $f_s(\mathcal{H}_{s}(j)) := j,  \forall j \in [0, size(\mathcal{H}_s))$}
    \State{Function $f_t$ maps items in $\mathcal{H}_t$ to labels: $f_t(\mathcal{H}_{t}(j)) := j,  \forall j \in [0, size(\mathcal{H}_t))$}
    \State{Update \emph{source} labels: $c_{l, s}^{(t+1)}(i) := f_s(h_{s}^{(t+1)}(i)))$,  $\forall i \in [0, n)$}
    \State{Update \emph{target} labels: $c_{l, t}^{(t+1)}(i) := f_t(h_{t}^{(t+1)}(i)))$,  $\forall i \in [0, n)$}
    \State{Update sets of labels: $L_{s}^{(t+1)} := \{c_{l, s}^{(t+1)}(i): i \in [0, n)\}$; $L_{t}^{(t+1)} := \{c_{l,t}^{(t+1)}(i): i \in [0, n)\}$}
    \If{$L_{s}^{(t+1)} = L_{s}^{(t)}$ and $L_{t}^{(t+1)} = L_{t}^{(t)}$}
    \State{\textbf{break}}
    \Else
    \State{$t \leftarrow t + 1$}
    \EndIf{}
    \EndWhile
    \State {\textbf{Output:}} $L_s^{(t)}$, $L_t^{(t)}$ 
\end{algorithmic}
\caption{Coloring of pairs of node labels in directed graphs} \label{alg:d-WL}
\end{algorithm}

\subsection{DiGAE encoder layer as  a message passing system}
\label{sec:digae-encoder-layer}
We can express our DiGAE encoder layer in a message-passing fashion; Algorithm \ref{alg:mp-d-GCN} details the operations performed at each node $i$. 

\begin{algorithm}
  \begin{algorithmic}
    \State {\textbf{Input:}} $\textbf{s}_i$, $\textbf{t}_i$; $\mathcal{N}^{+}(i)$, $\mathcal{N}^{-}(i)$;  $\textbf{W}_S$, $\textbf{W}_T$; $\alpha$, $\beta$

    \State{Set:  $\tilde{\mathcal{N}}^{+}(i) := \mathcal{N}^{+}(i) \cup \{i\}$ and $\texttt{deg}^{+}(i) := | \tilde{\mathcal{N}}^{+}(i) |$;
    $\tilde{\mathcal{N}}^{-}(i) := \mathcal{N}^{-}(i) \cup \{i\}$ and $\texttt{deg}^{-}(i) := | \tilde{\mathcal{N}}^{-}(i) |$}

  \State{Transform: $\textbf{s}_i \leftarrow \textbf{W}_S^{\top} \textbf{s}_i$;  $\textbf{t}_i \leftarrow \textbf{W}_T^{\top} \textbf{t}_i$;}

  \State{Scale: $\textbf{s}_i \leftarrow \frac{1}{(\texttt{deg}^{+}(i))^{\beta}} \; \textbf{s}_i$;
  $\textbf{t}_i \leftarrow \frac{1}{(\texttt{deg}^{-}(i))^{\alpha}} \; \textbf{t}_i$}

  \State{Send: $\texttt{send}(\textbf{s}_i, j)$ \emph{to all nodes} $j \in \tilde{\mathcal{N}}^{+}(i)$;
    $\texttt{send}(\textbf{t}_i, j)$ \emph{to all nodes} $j \in \tilde{\mathcal{N}}^{-}(i)$}

  \State{Receive:  $\textbf{s}_j \leftarrow \texttt{receive}(j)$ \emph{from all nodes} $j \in \tilde{\mathcal{N}}^{-}(i)$;
    $\textbf{t}_j \leftarrow \texttt{receive}(j)$ \emph{from all nodes} $j \in \tilde{\mathcal{N}}^{+}(i)$}
 
  \State{Sum: $\textbf{s}_i \leftarrow
    \sum_{j \in \tilde{\mathcal{N}}^{+}(i)}\textbf{t}_j$; 
    $\textbf{t}_i \leftarrow
    \sum_{j \in \tilde{\mathcal{N}}^{-}(i)}\textbf{s}_j$}

  \State{Scale: $\textbf{s}_i \leftarrow \frac{1}{(\texttt{deg}^{+}(i))^{\beta}} \; \textbf{s}_i$;
    $\textbf{t}_i \leftarrow \frac{1}{(\texttt{deg}^{-}(i))^{\alpha}} \; \textbf{t}_i$}
  
   \State {\textbf{Output:}} $\textbf{s}_i$, $\textbf{t}_i$
\end{algorithmic}
\caption{} \label{alg:mp-d-GCN}
\end{algorithm}

\subsection{CoraML and CiteSeer datasets}

\subsubsection{Impact of $\alpha$ and $\beta$ on the performance for DiGAE models}
\label{sec:grid-search}
Figure \ref{fig:grid-search} lists the mean AUC scores of the \emph{DiGAE} model for various 
combinations of the hyperparameters $(\alpha,\beta)$.
Numbers inside the boxes are mean AUC values collected during hyperparameter tuning, for the $\eta$ and $d$ values of the \emph{selected} model.

\begin{figure}[ht]
  \centering
  \includegraphics[width=0.45\textwidth]{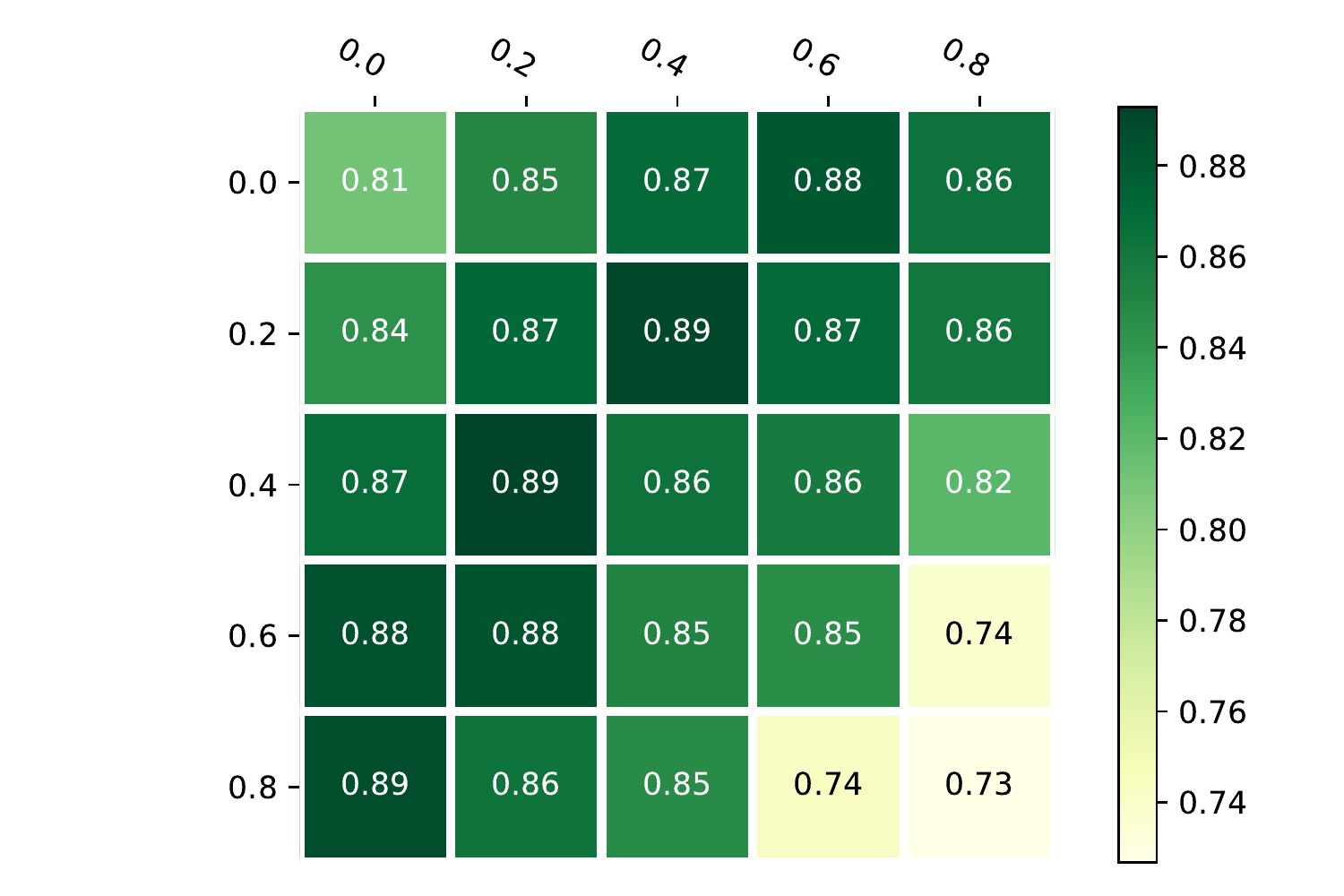}
  \includegraphics[width=0.45\textwidth]{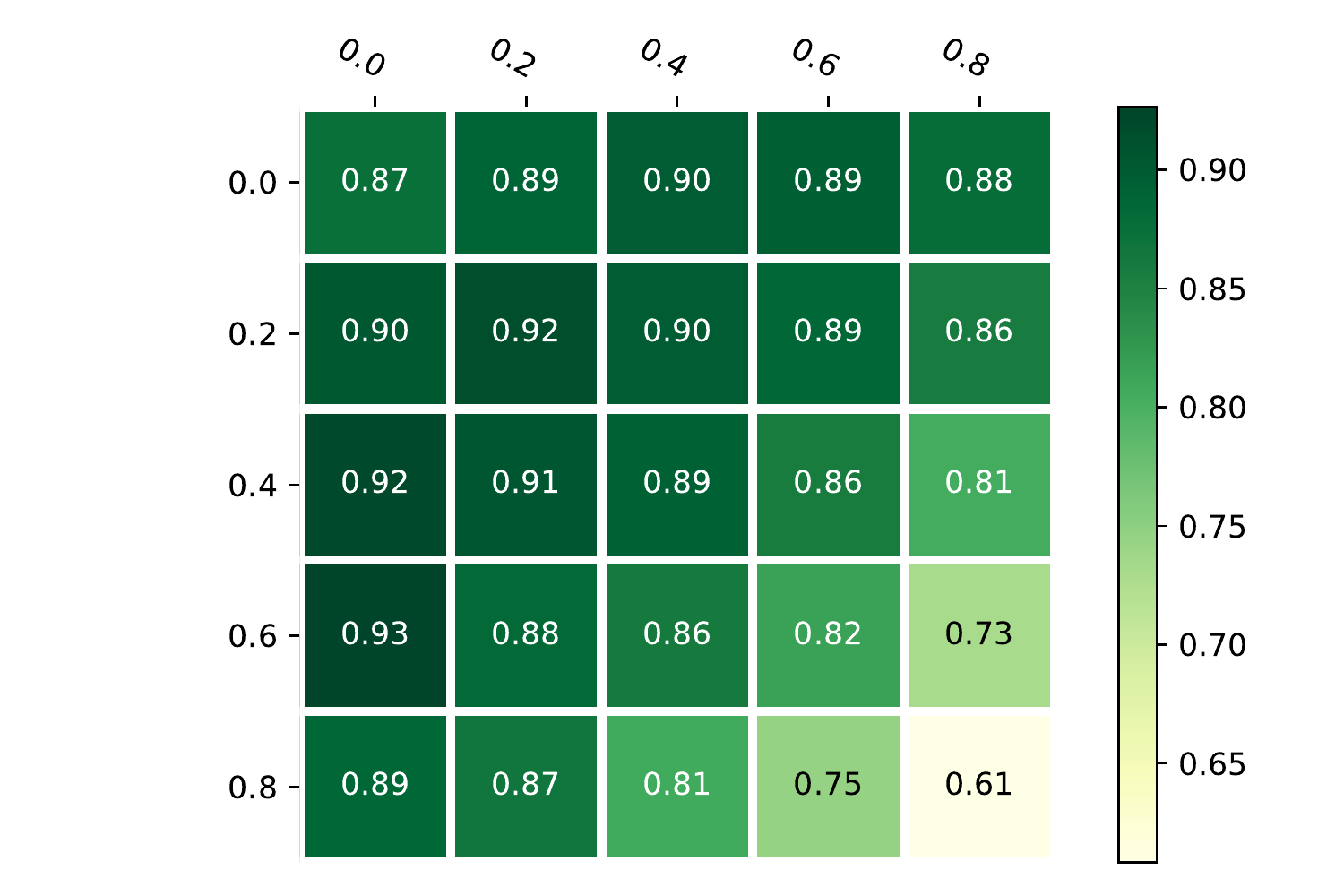}
  \caption{Grid search for CoraML (left) and CiteSeer (right) datasets to identify optimal 
    $\alpha$ (vertical) and $\beta$ (horizontal) for DiGAE model.
    Numbers inside the boxes indicate respective AUC scores.}
  \label{fig:grid-search}
\end{figure}

Figure \ref{fig:grid-search-single} is the analogous for Figure \ref{fig:grid-search}, but for the \emph{DiGAE-1L} model.

\begin{figure}[ht]
  \centering
  \includegraphics[width=0.45\textwidth]{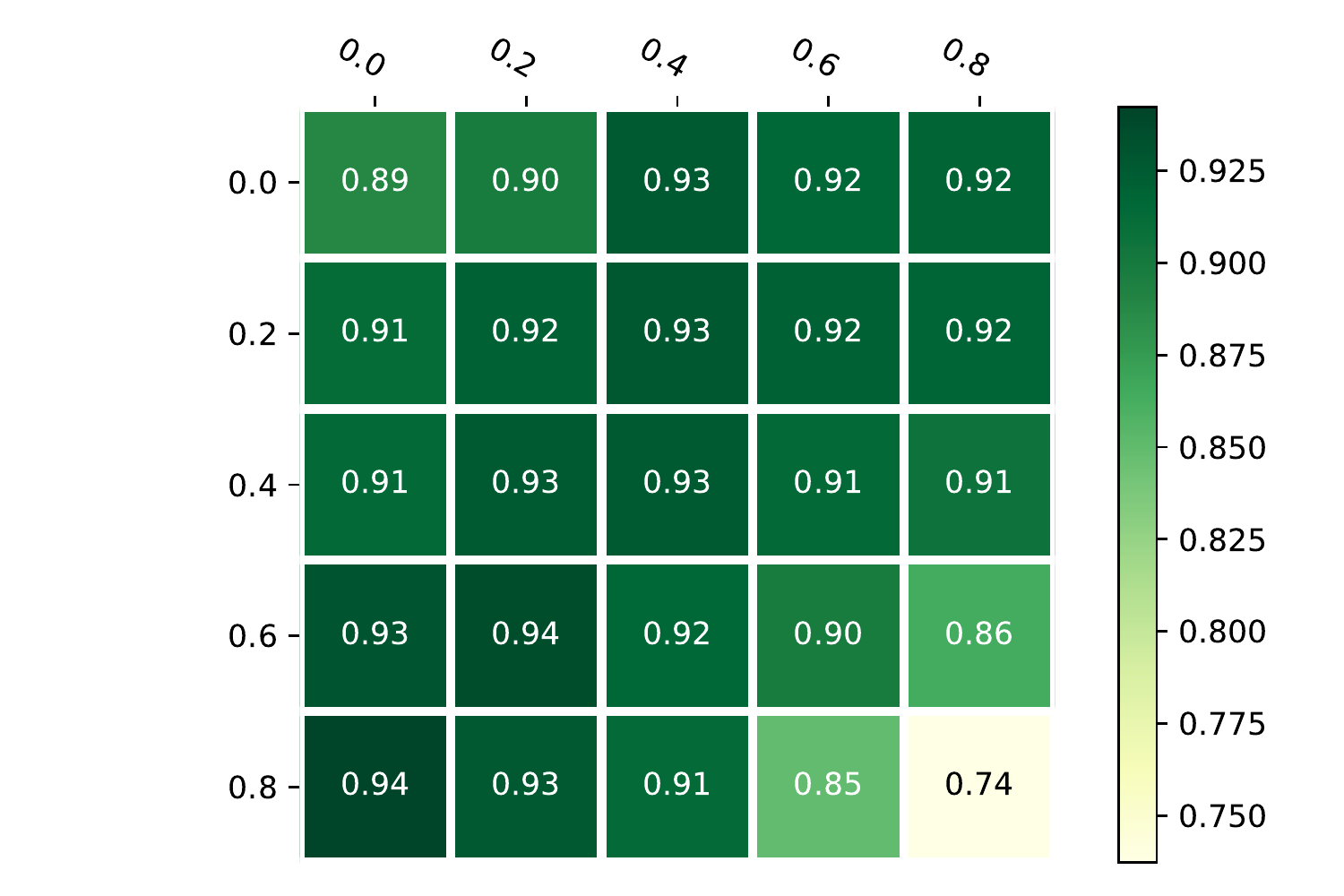}
  \includegraphics[width=0.45\textwidth]{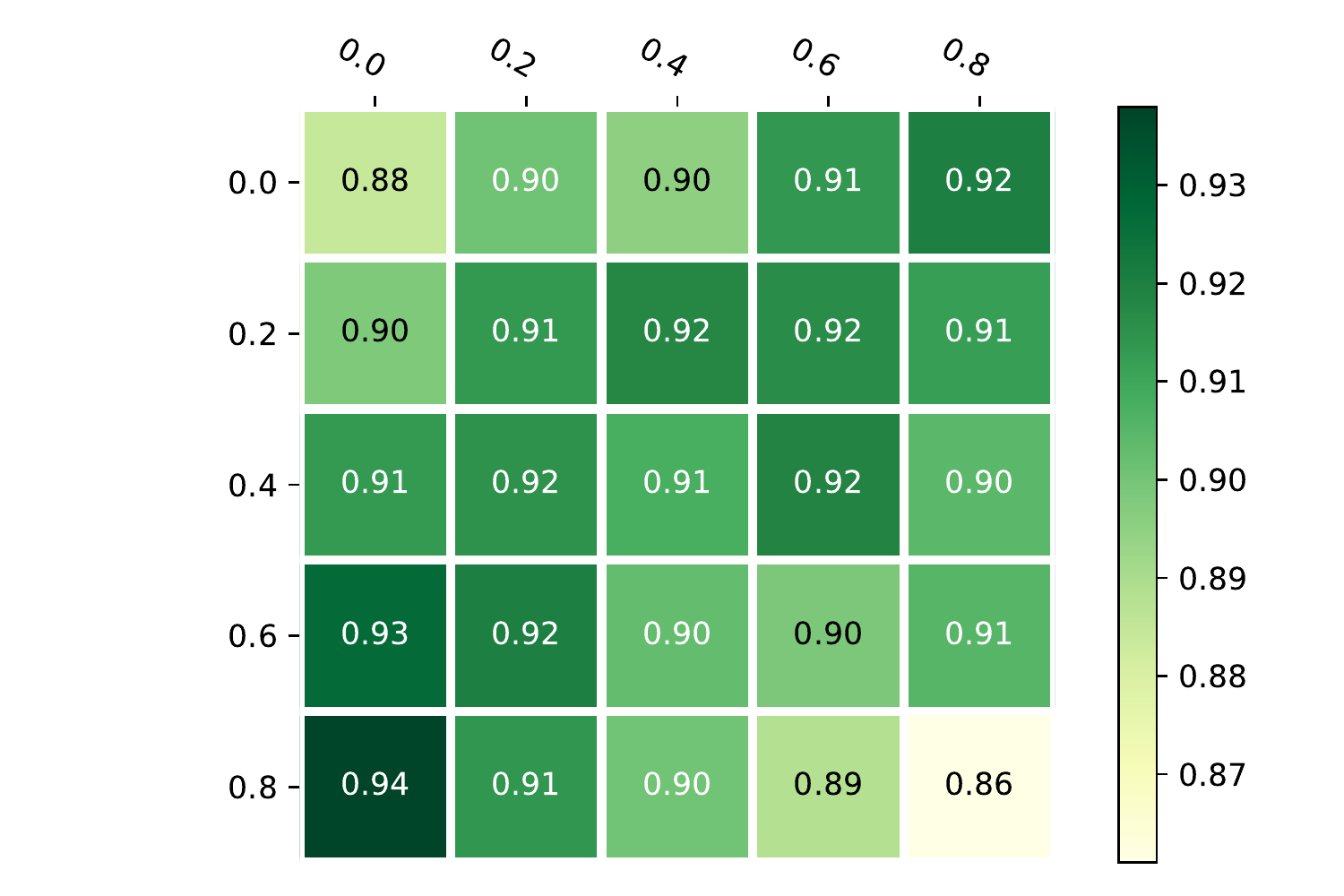}
  \caption{Grid search for CoraML (left) and CiteSeer (right) datasets to identify optimal 
    $\alpha$ (vertical) and $\beta$ (horizontal) for DiGAE-1L model.
    Numbers inside the boxes indicate respective AUC scores.}
  \label{fig:grid-search-single}
\end{figure}

\subsubsection{Truncated SVD baseline metrics}
Table \ref{tab:svd_table} details the metrics obtained in the truncated SVD baseline experiments.

\begin{table}[ht]
  \centering
  \caption{Truncated SVD baseline metrics (AUC and AP) for CoraML and CiteSeer datasets: two SVD implementations (standard and randomized) and
different encoding vector sizes: $k= 2, 4, 8, 16, 32, 64, 128$.}
\label{tab:svd_table}

\begin{tabular}{llrll}
\toprule
 Dataset &   Model &   k &            AUC &             AP \\
\midrule
  CoraML &     SVD &   2 & 74.41 +/- 1.45 & 75.05 +/- 1.47 \\
  CoraML &     SVD &   4 & 79.22 +/- 1.90 & 80.45 +/- 1.47 \\
  CoraML &     SVD &   8 & 79.94 +/- 1.60 & 81.78 +/- 1.31 \\
  CoraML &     SVD &  16 & 78.79 +/- 1.35 & 81.61 +/- 1.35 \\
  CoraML &     SVD &  32 & 77.91 +/- 1.71 & 81.84 +/- 1.47 \\
  CoraML &     SVD &  64 & 75.23 +/- 1.33 & 80.76 +/- 1.09 \\
  CoraML &     SVD & 128 & 70.79 +/- 1.70 & 78.03 +/- 1.12 \\
  \midrule
  CoraML & RandSVD &   2 & 72.62 +/- 2.64 & 76.11 +/- 2.02 \\
  CoraML & RandSVD &   4 & 75.85 +/- 1.75 & 78.91 +/- 1.35 \\
  CoraML & RandSVD &   8 & 78.09 +/- 1.08 & 81.36 +/- 1.15 \\
  CoraML & RandSVD &  \textbf{16} & 79.15 +/- 2.06 & \textbf{82.83 +/- 1.65} \\
  CoraML & RandSVD &  32 & 78.14 +/- 1.88 & 82.61 +/- 1.43 \\
  CoraML & RandSVD &  64 & 75.83 +/- 1.66 & 81.41 +/- 1.32 \\
  CoraML & RandSVD & 128 & 73.25 +/- 2.27 & 79.98 +/- 1.74 \\
  \midrule
CiteSeer &     SVD &   2 & 61.62 +/- 1.26 & 61.84 +/- 1.17 \\
CiteSeer &     SVD &   4 & 62.23 +/- 1.31 & 62.62 +/- 1.34 \\
CiteSeer &     SVD &   8 & 63.76 +/- 1.41 & 64.36 +/- 1.25 \\
CiteSeer &     SVD &  16 & 62.84 +/- 1.59 & 63.88 +/- 1.42 \\
CiteSeer &     SVD &  32 & 63.11 +/- 1.44 & 64.42 +/- 1.44 \\
CiteSeer &     SVD &  64 & 62.51 +/- 1.60 & 64.31 +/- 1.39 \\
CiteSeer &     SVD & 128 & 63.51 +/- 1.93 & 65.81 +/- 1.64 \\
\midrule
  CiteSeer & RandSVD &   2 & 63.45 +/- 1.61 & 66.13 +/- 1.57 \\
CiteSeer & RandSVD &   4 & 64.52 +/- 1.81 & 67.39 +/- 1.53 \\
CiteSeer & RandSVD &   8 & 67.34 +/- 1.67 & 69.72 +/- 1.38 \\
CiteSeer & RandSVD &  \textbf{16} & 66.86 +/- 1.54 & \textbf{69.88 +/- 1.23} \\
CiteSeer & RandSVD &  32 & 66.33 +/- 2.05 & 69.76 +/- 1.56 \\
CiteSeer & RandSVD &  64 & 65.61 +/- 2.50 & 69.40 +/- 1.87 \\
CiteSeer & RandSVD & 128 & 63.94 +/- 1.47 & 68.79 +/- 1.25 \\
\bottomrule
\end{tabular}
\end{table}

\subsubsection{HOPE baselines}
For comparison and for the standard Katz proximity ${(\textbf{I} - \beta \textbf{A})}^{-1} \beta \textbf{A}$ in HOPE \cite{ou2016asymmetric},
we computed AUC and AP metrics for $k \in \{2^{i}|i=1,2,\ldots, 7\}$, $\beta=0.02$, $20$ random splits, for CoraML and CiteSeer datasets.
Table \ref{tab:task_1-katz-index-proximity} details the metrics obtained in the experiments.
\begin{table}[ht]
  \begin{center}
    \caption{HOPE baseline metrics (AUC and AP) for CoraML and CiteSeer datasets and for
different encoding vector sizes: $k= 2, 4, 8, 16, 32, 64, 128$. Katz Index proximity has its $\beta$ parameter set to $0.02$.}
\label{tab:task_1-katz-index-proximity}
\begin{tabular}{llrll}
\toprule
 Dataset &     Model &   k &            AUC &             AP \\
\midrule
  CoraML & HOPE/Katz &   2 & 71.24 +/- 1.09 & 72.08 +/- 1.23 \\
  CoraML & HOPE/Katz &   4 & 79.27 +/- 1.53 & 80.18 +/- 1.35 \\
  CoraML & HOPE/Katz &   8 & 81.38 +/- 1.50 & 82.49 +/- 1.47 \\
  CoraML & HOPE/Katz &  16 & 79.47 +/- 1.40 & 82.05 +/- 1.18 \\
  CoraML & HOPE/Katz &  32 & 80.77 +/- 1.50 & \textbf{83.87 +/- 1.23} \\
  CoraML & HOPE/Katz &  64 & 78.08 +/- 1.41 & 82.45 +/- 1.16 \\
  CoraML & HOPE/Katz & 128 & 75.00 +/- 1.95 & 80.91 +/- 1.38 \\
  \midrule
CiteSeer & HOPE/Katz &   2 & 60.64 +/- 1.27 & 60.75 +/- 1.25 \\
CiteSeer & HOPE/Katz &   4 & 62.14 +/- 1.66 & 62.56 +/- 1.58 \\
CiteSeer & HOPE/Katz &   8 & 63.66 +/- 1.49 & 64.06 +/- 1.46 \\
CiteSeer & HOPE/Katz &  16 & 64.17 +/- 1.69 & 65.05 +/- 1.62 \\
CiteSeer & HOPE/Katz &  32 & 63.72 +/- 1.47 & 65.10 +/- 1.44 \\
CiteSeer & HOPE/Katz &  64 & 64.33 +/- 2.24 & 66.03 +/- 1.89 \\
CiteSeer & HOPE/Katz & 128 & 65.07 +/- 1.36 & \textbf{67.29 +/- 1.27} \\
\bottomrule
\end{tabular}
\end{center}
\end{table}

\subsection{Results on WebKB collection datasets}
We experiment with the following datasets from the WebKB collection \footnote{\url{http://www.cs.cmu.edu/afs/cs.cmu.edu/project/theo-11/www/wwkb/}}:
``Wisconsin'', ``Texas'' and ``Cornell''; their properties are in Table \ref{tab:webkb_datasets}.
In WebKB datasets, nodes represent Web pages and are manually classified into the five categories: student, project, course, staff, and faculty;
edges are hyperlinks between them. Node
features are the bag-of-words representation of the Web pages.
We apply the same pipeline as for CoraML and CiteSeer citation datasets.
In Table \ref{tab:task_1_webkb_features} we report the peformance metrics for the case the node features are input in model training.
For the feature-less configurations, performance of baselines is reported in Table \ref{tab:task_1_webkb}.

\begin{table}
  \begin{center}
\caption{WebKB datasets.}   \label{tab:webkb_datasets}
  \begin{tabular}{lccc}
    \hline
    Dataset   &   Number of nodes &   Number of edges &   Number of features \\
    \hline
    Texas      &               183 &              325  &                 1703 \\
    Cornell    &               183 &              298  &                 1703 \\
    Wisconsin  &               251 &              515  &                 1703 \\
    \hline
  \end{tabular}
  \end{center}
\end{table}

\begin{table}[ht]
  \begin{center}
  \caption{General directed link prediction for the \emph{WebKB} graphs (feature-based configurations). }
\label{tab:task_1_webkb_features}
  \begin{tabular}{lllll}
    \toprule
    Dataset &                Model &             AUC &              AP &          Time (secs) \\
\midrule
    Wisconsin &         DiGAE (ours) &  73.95 +/- 7.55 &  79.64 +/- 6.31 & 1.12 +/- 0.04 \\
    Wisconsin & DiGAE 1-Layer (ours) &  \textbf{81.14 +/- 7.36} &  \textbf{85.60 +/- 5.03} & 0.82 +/- 0.03 \\
    Wisconsin &          Gravity GAE &  73.37 +/- 4.19 &  76.73 +/- 3.97 & 0.70 +/- 0.10 \\
    Wisconsin &    Source/Target GAE &  72.73 +/- 5.31 &  78.40 +/- 4.17 & 0.56 +/- 0.04 \\
    Wisconsin &         Standard GAE &  67.84 +/- 5.02 &  75.34 +/- 4.28 & 0.51 +/- 0.10 \\
    \midrule                           
    Texas &         DiGAE (ours) & 64.79 +/- 17.24 & 74.38 +/- 12.99 & 0.93 +/- 0.03 \\
    Texas & DiGAE 1-Layer (ours) &  \textbf{70.58 +/- 8.16} &  \textbf{76.97 +/- 7.82} & 0.69 +/- 0.03 \\
    Texas &          Gravity GAE &  64.14 +/- 7.98 &  71.38 +/- 7.41 & 0.56 +/- 0.08 \\
    Texas &    Source/Target GAE &  63.50 +/- 7.41 &  68.42 +/- 6.72 & 0.47 +/- 0.05 \\
    Texas &         Standard GAE &  50.80 +/- 7.93 &  62.28 +/- 5.60 & 0.46 +/- 0.05 \\
    \midrule
    Cornell &         DiGAE (ours) & \textbf{69.69 +/- 12.07} &  \textbf{77.14 +/- 9.77} & 0.90 +/- 0.03 \\
    Cornell & DiGAE 1-Layer (ours) & 61.96 +/- 10.27 & 70.69 +/- 10.26 & 0.65 +/- 0.02 \\
    Cornell &          Gravity GAE &  66.24 +/- 4.84 &  69.10 +/- 5.01 & 0.59 +/- 0.09 \\
    Cornell &    Source/Target GAE &  62.82 +/- 5.75 &  69.42 +/- 5.25 & 0.47 +/- 0.05 \\
    Cornell &         Standard GAE &  57.49 +/- 7.78 &  68.72 +/- 6.54 & 0.47 +/- 0.05 \\
\bottomrule
  \end{tabular}
  \end{center}
\end{table}

\begin{table}[ht]
  \caption{General directed link prediction for the \emph{WebKB} graphs (feature-less configurations).}
  \label{tab:task_1_webkb}
  \begin{center}
  \begin{tabular}{llll}
    \toprule
    Dataset &             Model &            AUC &             AP \\
\midrule
    Wisconsin &       Gravity GAE & 67.04 +/- 5.35 & 74.41 +/- 4.42 \\
    Wisconsin & Source/Target GAE & 72.46 +/- 5.89 & 77.05 +/- 5.68 \\
    Wisconsin &      Standard GAE & 65.89 +/- 5.28 & 74.00 +/- 4.13 \\
    \midrule
    Texas &       Gravity GAE & 65.53 +/- 7.30 & 70.77 +/- 6.47 \\
    Texas & Source/Target GAE & 62.37 +/- 5.24 & 67.58 +/- 4.13 \\
    Texas &      Standard GAE & 53.81 +/- 6.61 & 64.24 +/- 5.27 \\
    \midrule
    Cornell &       Gravity GAE & 69.58 +/- 5.26 & 73.41 +/- 5.07 \\
    Cornell & Source/Target GAE & 65.14 +/- 5.71 & 69.90 +/- 5.18 \\
    Cornell &      Standard GAE & 60.26 +/- 7.96 & 69.21 +/- 6.91 \\
    \bottomrule
  \end{tabular}
  \end{center}
\end{table}

Likewise for citation datasets,  both DiGAE models perform better in mean AUC and AP for all WebKB networks.
Note however that performance metrics for any model over different splits for these small graphs have large standard deviations, which does not provide a clean performance ranking of the models.

\subsection{Results on Pubmed dataset}

We experiment with the Pubmed dataset (19,717 nodes, 44,338 edges, 500 features) which is a citation network pertaining to diabetes classified into one of three classes (``Diabetes Mellitus, Experimental'', ``Diabetes Mellitus Type 1'', ``Diabetes Mellitus Type 2'') \footnote{\url{https://linqs.soe.ucsc.edu/data}}.

We employ the same grid search strategy for hyperparameter tuning and training setup as for other citation datasets.
We report on the mean AUC and AP metrics and their respective timings over a series of ten experiments for each selected model, with one hot-encoding of the nodes (feature-less configurations), in Table \ref{tab:task_1-pubmed-feature-less}. DiGAE models consistently and significantly outperform other baselines by margins which are in the $7\%$ to $13\%$ range for mean AUC and in the $4\%$ to $8\%$ range for mean AP; they are also considerably faster. DiGAE-1L in particular is the fastest overall by a factor close or exceeding $\times 10$ and performs best overall in the mean AUC and AP metrics.

\begin{table}[ht]
  \begin{center}
    \caption{General directed link prediction for \emph{Pubmed} dataset (feature-less configurations).}
\label{tab:task_1-pubmed-feature-less}

\begin{tabular}{lllll}
\toprule
Dataset &                 Model &             AUC &              AP &             Time (secs) \\
\midrule
 Pubmed &          DiGAE (ours) &  92.09 +/- 1.92 &  91.08 +/- 1.22 &   98.62 +/- 2.75 \\
 Pubmed &  DiGAE 1-Layer (ours) &  \textbf{94.35 +/- 0.37} &  \textbf{93.16 +/- 0.48} &   45.49 +/- 0.48 \\
 Pubmed &           Gravity GAE &  84.28 +/- 0.59 &  86.60 +/- 0.58 &  560.09 +/- 9.94 \\
 Pubmed &     Source/Target GAE &  85.01 +/- 1.36 &  87.28 +/- 0.60 &  418.40 +/- 2.13 \\
 Pubmed &          Standard GAE &  81.43 +/- 0.32 &  85.88 +/- 0.31 &  421.56 +/- 5.95 \\
\bottomrule
\end{tabular}

\end{center}
\end{table}

\subsubsection{Impact of $\alpha$ and $\beta$ on the performance for DiGAE models}
In Figure \ref{fig:pubmed-grid-search-feature-less} we list the mean AUC scores of the \emph{DiGAE} models for various 
combinations of hyperparameters $(\alpha,\beta)$ for the Pubmed dataset. Numbers inside the boxes are mean AUC values collected during hyperparameter tuning, for the $\eta$ and $d$ values of the \emph{selected} model.
We confirm the flexibility to further optimize performance metrics by suitably tuning node degrees' exponentiation parameters $\alpha$ and $\beta$.

\begin{figure}[ht]
  \centering
  \includegraphics[width=0.45\textwidth]{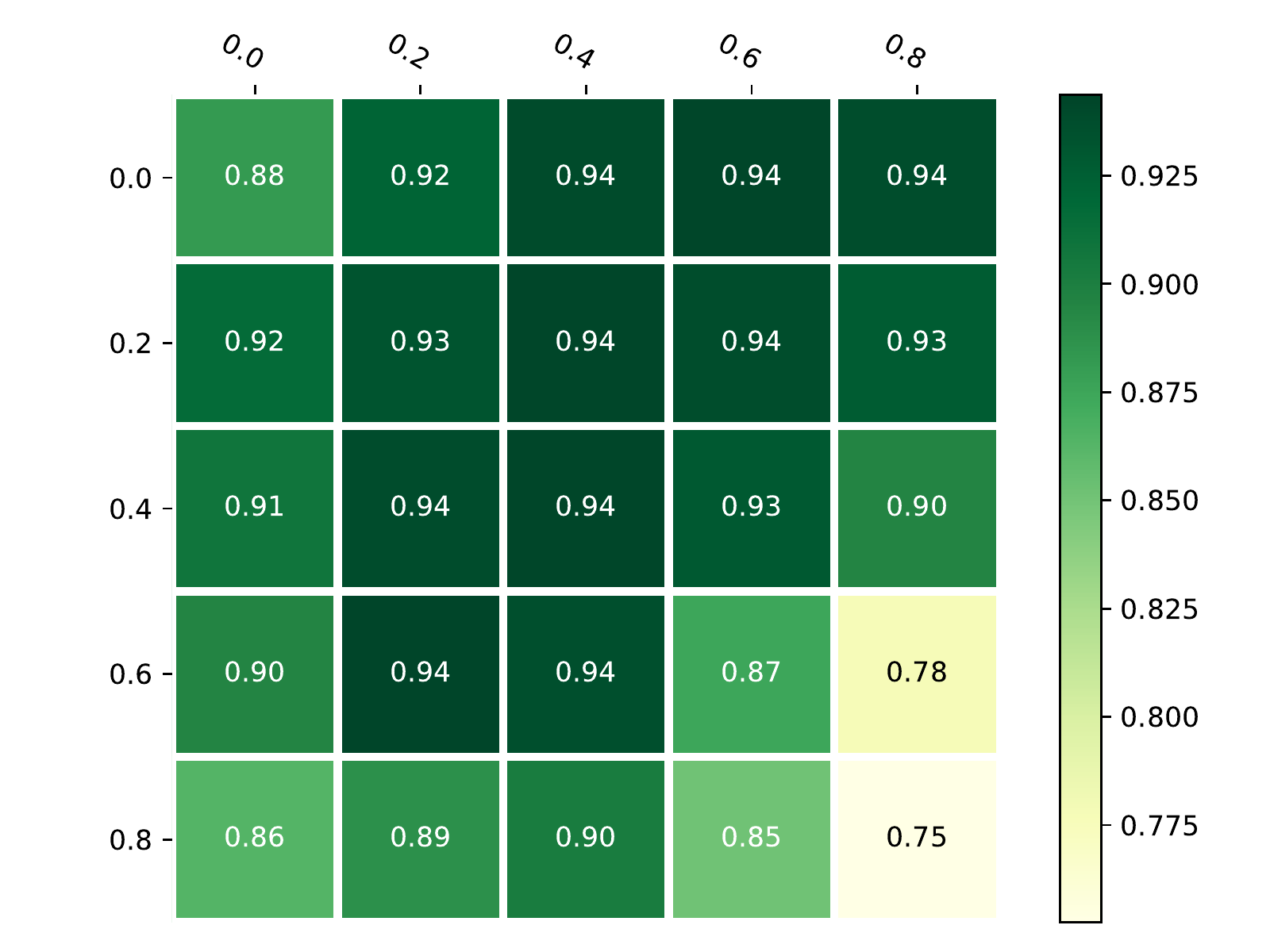}
  \includegraphics[width=0.45\textwidth]{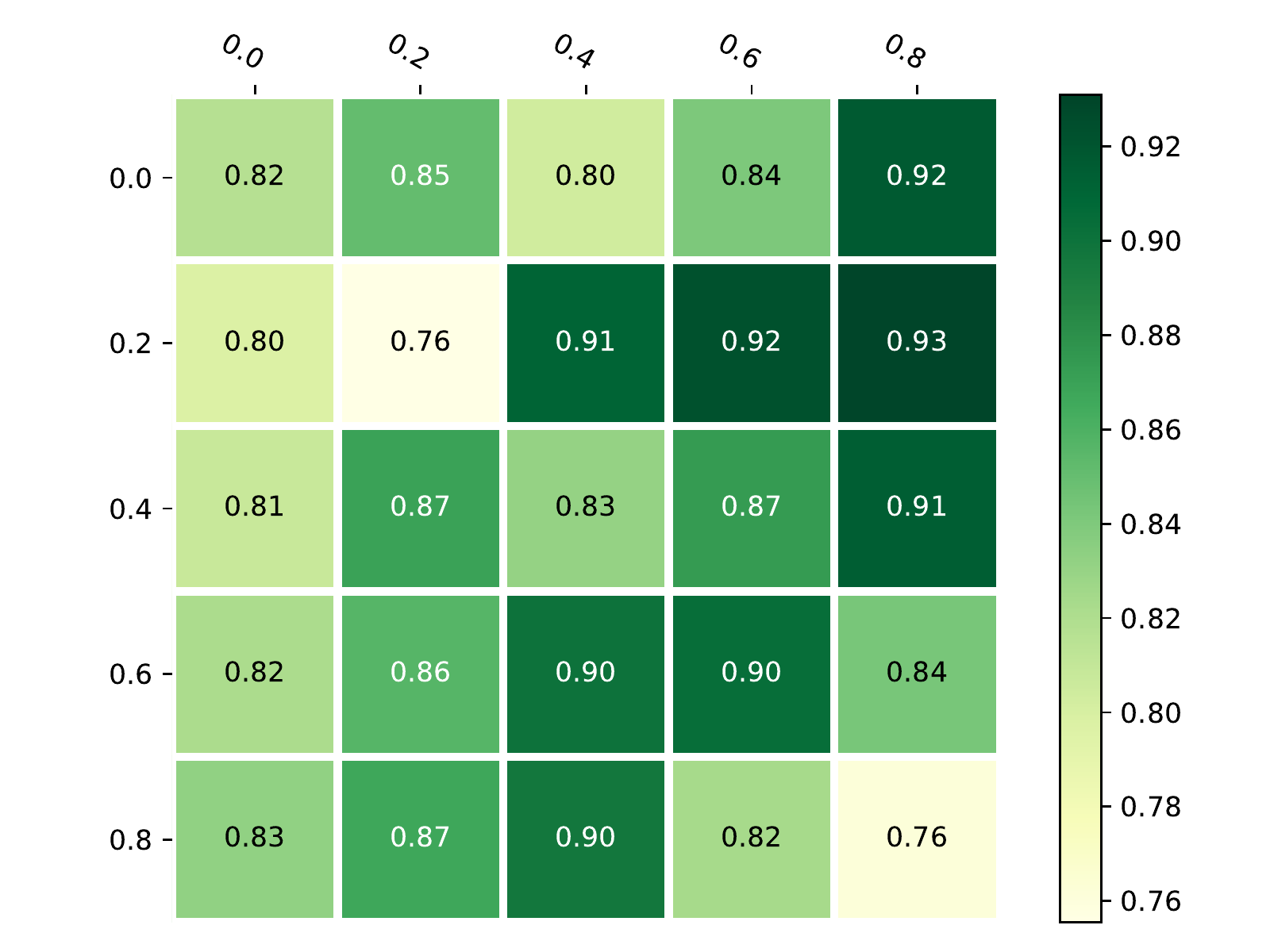}
  \caption{Grid search for Pubmed dataset to identify optimal 
    $\alpha$ (vertical) and $\beta$ (horizontal) for DiGAE-1L (left pane) and DiGAE (right pane) models.
    Numbers inside the boxes indicate respective AUC scores.}
  \label{fig:pubmed-grid-search-feature-less}
\end{figure}

\subsection{A note on the role of parameters $\alpha$ and $\beta$.}
By varying exponent parameters $\alpha$ and $\beta$ we modify the spectrum (singular values) of our multiplication matrices $\hat{\textbf{A}}$, $\hat{\textbf{{A}}}^{\top}$ (please see Equations (6) in the text). In particular for smaller parameter values, the scaling role of the node degrees $d$ is supressed (and in the limit of $0$ it vanishes: $1/d^{0}=1$) and the spectrum max shifts to larger values while for larger parameters the spectrum range shrinks. Spectrum shrinking will typically lead to smaller magnitudes for source and target encodings from Equations (6) (and their respective inner products) and so the sigmoid function in the decoder can fail to predict links: this is consistent with the particularly low performance metrics when both $\alpha$ and $\beta$ are large. Please see the spectrum for CoraML's $\hat{\textbf{A}}$ for $\alpha,\beta\in\{0.3, 0.5, 0.8\}, \alpha=\beta$ in Figure \ref{fig:singular_values}.
\begin{figure}[ht]
  \centering
  \includegraphics[width=0.7\textwidth]{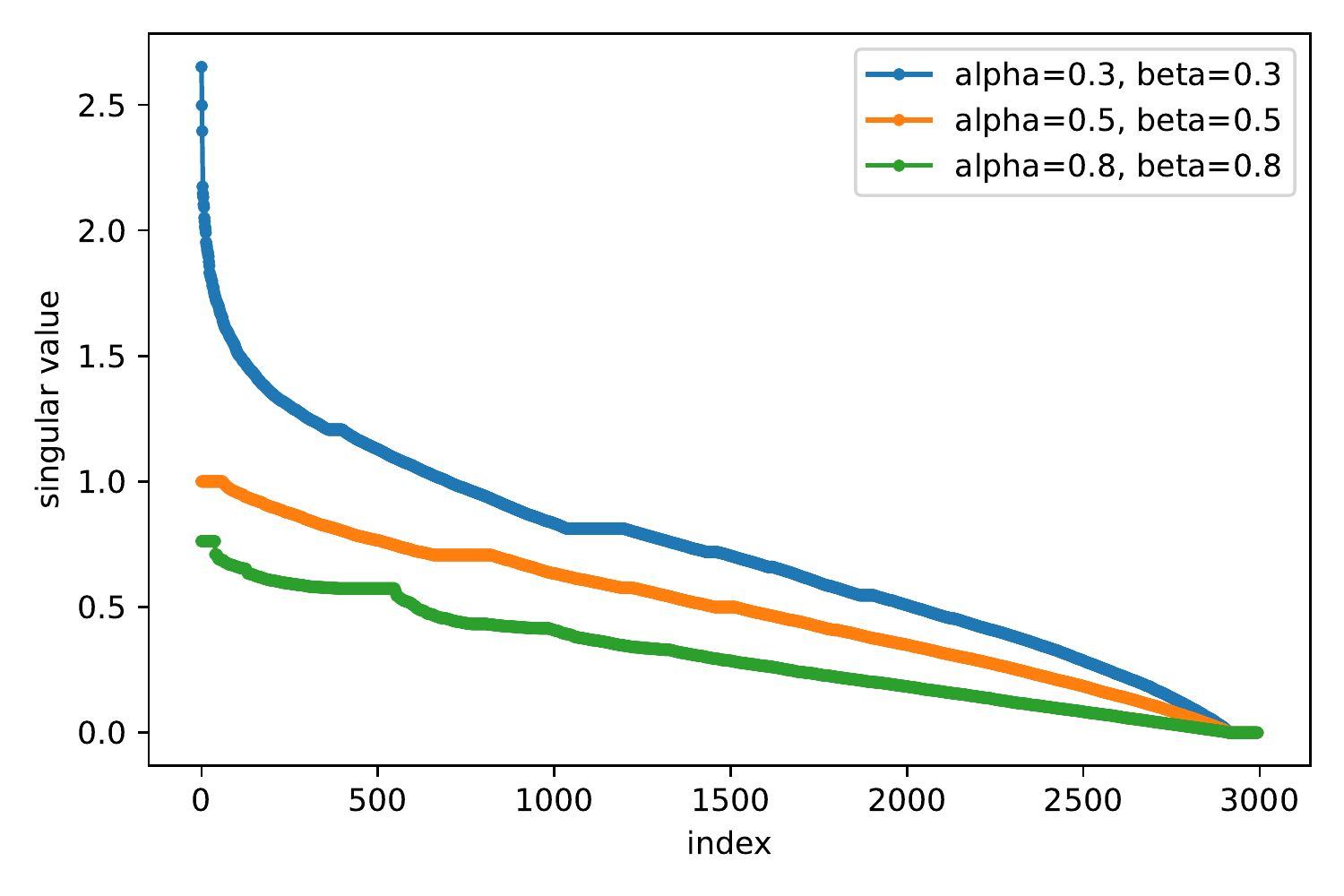}
  \caption{Spectrum for CoraML's $\hat{\textbf{A}}$ for $\alpha,\beta\in\{0.3, 0.5, 0.8\}, \alpha=\beta$.}
  \label{fig:singular_values}
\end{figure}
From the local smoothing perspective around a node $i$, smaller $\alpha$'s reinforce the role of "authority" nodes $j$ pointed by $i$ (i.e. nodes with large indegree $\texttt{deg}^{-}(j)$) in updating source encoding $\textbf{s}_i$ from $\textbf{t}_j$'s (contributing terms analogous to 
$\frac{1}{ \texttt{deg}^{-}(j)^{\alpha}} \textbf{t}_j$). Similarly,  smaller $\beta$'s reinforce the role of "hub" nodes $j$ pointing to $i$ (i.e. nodes with large outdegree $\texttt{deg}^{+}(j)$) in updating target encoding $\textbf{t}_i$ from $\textbf{s}_j$'s (contibuting terms analogous to 
$\frac{1}{\texttt{deg}^{+}(j)^{\beta}} \textbf{s}_j$). For a deeper theoretical analysis of the impact of the parameters on the spectum, the simplified linear case (as in ``Simplifying Graph Convolutional Networks'', Wu et al (2019), there for the impact of self-links) could be a starting point and this is deferred for future research.

\newpage
\clearpage


\end{document}